\documentclass[11pt]{article}
\usepackage[]{acl}

\usepackage{hyperref,times,url}
\usepackage{other}
\usepackage{cancel}
\usepackage{inconsolata}
\usepackage[normalem]{ulem}
\usepackage{microtype}
\usepackage[T1]{fontenc}
\usepackage[utf8]{inputenc}

\usepackage[inline]{enumitem}
\setitemize{noitemsep,topsep=0pt,parsep=0pt,partopsep=0pt,leftmargin=*}
\setenumerate{noitemsep,topsep=0pt,parsep=0pt,partopsep=0pt,leftmargin=*}

\title{Prefix Parsing is Just Parsing}

\newcommand{\ETH}[0]{{\normalfont \textsuperscript{1}}}
\newcommand{\MCGILL}[0]{{\normalfont \textsuperscript{2}}}
\newcommand{\CIFAR}[0]{{\normalfont\textsuperscript{3}}}
\newcommand{\MILA}[0]{\normalfont{\textsuperscript{4}}}
\newcommand{\CHIFRO}[0]{{\normalfont \textsuperscript{5}}}

\author{\textbf{\href{https://clementep.github.io/}{\textcolor{black}{Clemente Pasti}}}\ETH\CHIFRO\quad
\textbf{\href{https://opedal.github.io/}{\textcolor{black}{Andreas Opedal}}}\ETH\quad
\textbf{\href{https://todonnell.github.io/}{\textcolor{black}{Timothy J.\@ O'Donnell}}}\MCGILL\CIFAR\MILA\CHIFRO \\
\textbf{\href{https://rycolab.io/authors/ryan/}{\textcolor{black}{Ryan Cotterell}}}\ETH\quad
\textbf{\href{https://timvieira.github.io}{\textcolor{black}{Tim Vieira}}}\ETH
\\
\texttt{
\{\href{mailto:clemente.pasti@inf.ethz.ch}{clemente.pasti}, \href{mailto:andreas.opedal@inf.ethz.ch}{andreas.opedal}, \href{mailto:ryan.cotterell@inf.ethz.ch}{ryan.cotterell}\}@inf.ethz.ch} \\
\texttt{\href{mailto:timothy.odonnell@mcgill.ca}{timothy.odonnell@mcgill.ca}}\quad
\texttt{\href{mailto:tim.f.vieira@gmail.com}{tim.f.vieira@gmail.com}} 
 \\
\ETH ETH Z\"urich\quad
\MCGILL McGill University\quad
\CIFAR Canada CIFAR AI Chair\quad
\MILA Mila\quad
\CHIFRO CHI-FRO
}

\addtocontents{toc}{\protect\setcounter{tocdepth}{-10000}}

\begin{document}

\maketitle

\begin{abstract}
Prefix parsing asks whether an input prefix can be extended to a complete string generated by a given grammar. In the weighted setting, it also provides prefix probabilities, which are central to context-free language modeling, psycholinguistic analysis, and syntactically constrained generation from large language models. We introduce the \emph{prefix grammar transformation}, an efficient reduction of prefix parsing to ordinary parsing. Given a grammar, our method constructs another grammar that generates exactly the prefixes of its original strings. Prefix parsing is then solved by applying any ordinary parsing algorithm on the transformed grammar without modification. The reduction is both elegant and practical: the transformed grammar is only a small factor larger than the input, and any optimized implementation can be used directly, eliminating the need for bespoke prefix-parsing algorithms. We also present a strategy---based on algorithmic differentiation---for computing the next-token weight vector, i.e., the prefix weights of \emph{all} one-token extensions, enabling efficient prediction of the next token. Together, these contributions yield a simple, general, and efficient framework for prefix parsing.\looseness=-1
\begin{center}
\vspace{-2.5 mm}
\raisebox{-0.25em}{\includegraphics[height=1em]{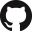}}\,
\texttt{\href{https://github.com/genlm/prefix-acl-26}{genlm/prefix-acl-26}}
\end{center}
\end{abstract}
\vspace{-4mm}
\section{Introduction}
\label{sec:introduction}
Parsing a string with respect to a context-free grammar (CFG) means determining whether---and with what weight---the grammar generates that string. \emph{Prefix} parsing asks the closely related question of whether an input prefix can be extended to a complete string generated by the grammar. In the probabilistic setting,\footnote{\label{footnote:semiring-extension}More generally, our method applies to grammars whose weights lie in any commutative semiring \citep[see, e.g.,][]{goodman-1999-semiring}. This includes the boolean semiring, which is useful for enforcing grammaticality constraints in language models \citep[e.g.,][]{shin-etal-2021-constrained,poesia-etal-2022-synchromesh,loula2025syntactic}.\looseness=-1} these tasks correspond, respectively, to computing (i) the probability that the grammar generates a given string, and (ii) the total probability of strings beginning with a given prefix.\looseness=-1

Existing prefix parsers are typically tied to particular parsing algorithms. For example, \citet{jelinek-lafferty-1991-computation} adapted CKY \citep{Cocke,Kasami,YOUNGER},
while \citet{stolcke-1995-efficient} adapted \citeposs{Earley} algorithm.\footnote{\citet{nowak-cotterell-2023-fast} developed a prefix parser based on CKY with the \emph{hook trick} \citep{eisner-blatz-2007}.} Although these algorithms require bespoke and surprisingly involved modifications to the base parser, they also share notable similarities. This raises a natural question: is there a general-purpose reduction from prefix parsing to parsing, one that would let us adapt \emph{other} parsers to the prefix setting, e.g., \citeposs{valiant1975cfg} subcubic parsing algorithm, GPU implementations,\footnote{E.g., \citet{stanojevic-sartran-2023-synjax}, \citet{rush-2020-torch}, \citet{yi-etal-2011-efficient}, \citet{dunlop-thesis}, and \citet{canny-etal-2013-multi}.} and fast CPU parsers?\footnote{E.g., \citet[][]{marpa} and \citet{dunlop-thesis}.}
We show that such a reduction exists. Given any CFG, we produce a new CFG---the \emph{prefix grammar}---that generates exactly the weighted prefix language of the original, and is only a small constant factor larger. 
Consequently, \emph{any} parsing algorithm serves as a prefix parser off-the-shelf: its input-length dependence is preserved, with the entire overhead confined to grammar-structural factors (\cref{thm:prefix-parsing-runtime}). How much those factors grow depends on the parser's preprocessing pipeline; we analyze them for \IncrCKY and \Earley (\cref{sec:app-prefix-parsers}).\footnote{We found that our implementation of Earley's algorithm with the prefix grammar is comparable to \citeposs{luong2013-parsing-discourse} implementation of Stolcke's algorithm (\cref{sec:EarleyX}).}
Moreover, the prefix grammar can be viewed as the composition of the original grammar with a two-state prefix transducer (\cref{sec:composition}), a perspective that extends naturally to other grammar formalisms.\looseness=-1

\begin{figure}[t]
\centering
\includegraphics[width=1 \linewidth]{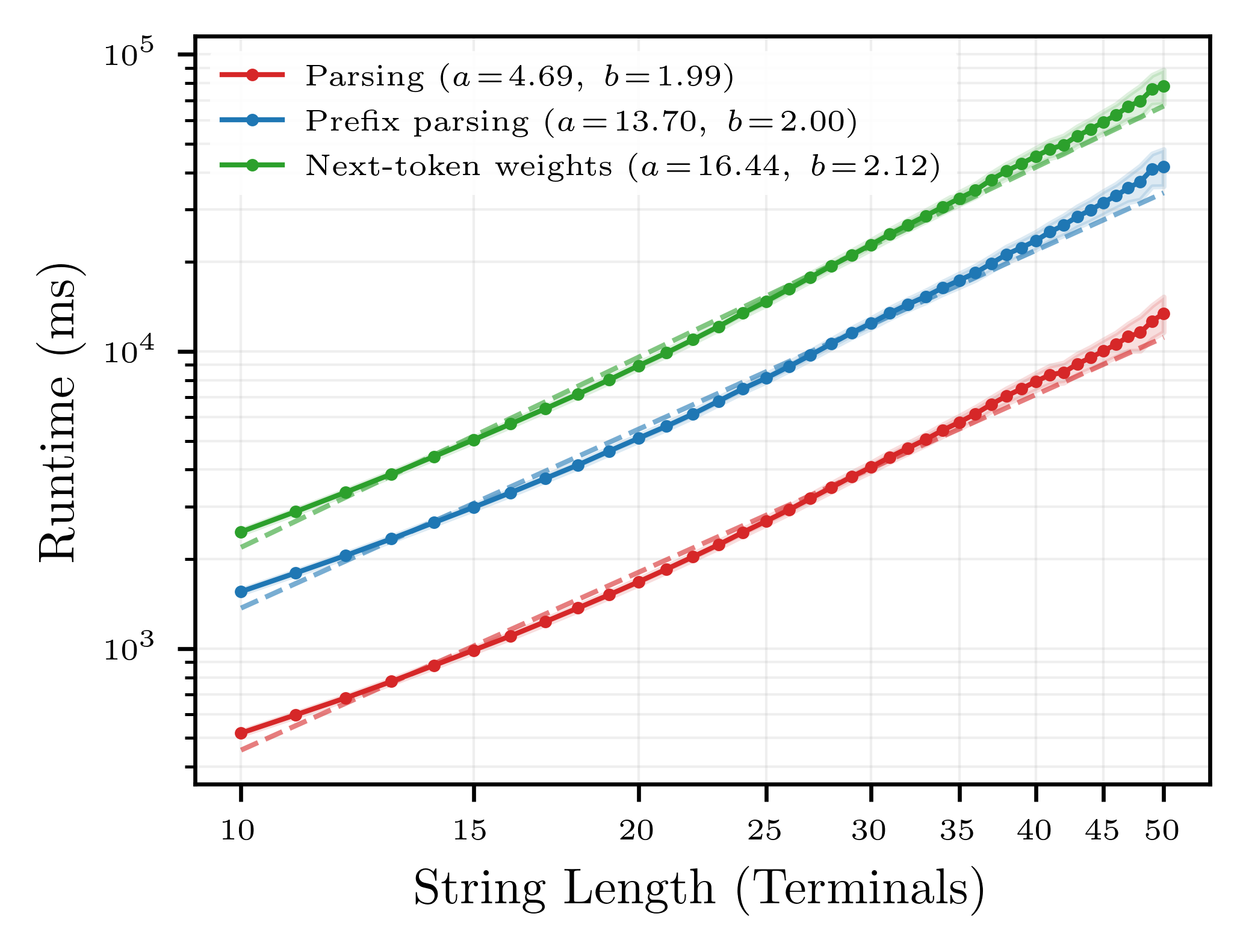}
\vspace{-9 mm}
\caption{\footnotesize 
We compare parsing, prefix parsing, and the next-token weight vector algorithm on WSJ 5000 \citep[\texttt{EarleyX}]{luong2013-parsing-discourse}, a large grammar with $35{,}016$ rules, $448$ nonterminals, and a total size of $116{,}667$. The underlying parser is Earley's (see \cref{alg:earley}), paired with its next-token variant (\cref{sec:next-token-algorithms,sec:app-next-token-earley}). Results are shown for $500$ strings on a log--log plot; to obtain data across all string lengths, we include the runtime for each prefix of every string. Shaded regions represent $95\%$ confidence intervals for the mean runtime at each string length.
To estimate runtime complexity, we fit a power law $r(N) = aN^b$ via least-squares regression on the log--log scale: $\log r(N) = \log a + b\log N$. Dashed lines show the fitted models. All three algorithms exhibit similar fitted complexity exponents ($1.99$, $2$, and $2.12$ respectively), confirming that the $N$-dependence is essentially independent of grammar size. Prefix parsing is roughly $2.9\times$ slower than parsing (fitted coefficient $a = 30.9$ vs.\ $12.1$), closely tracking the prefix grammar's ${\approx}\, 2.86\times$ size increase (binarize, then prefix). Both observations are predicted by \cref{thm:prefix-parsing-runtime}: the $N$-dependence is preserved, and the overhead is absorbed into the grammar-size factor. The next-token algorithm is slower than prefix parsing by only a factor of ${\approx}\, 1.2$, well within the constant-factor bound (i.e., $4\times$) of the next-token meta-theorem (\cref{thm:meta-theorem}).\looseness=-1
}
\looseness=-1
\vspace{-5 mm}
\label{fig:runtime-comparison-three-algorithms}
\end{figure}

Finally, we develop efficient algorithms for computing the \emph{next-token weights} of a string---the prefix weights of all one-token extensions. Surprisingly, our strategy---which leverages algorithmic differentiation \citep[see, e.g.,][]{griewank, eisner-2016-inside} and is closely related to the outside algorithm \citep{insideOutside}---allows us to compute the next-token weights of a string in the same asymptotic runtime as prefix-parsing the string itself. On the other hand, a na\"ive approach would incur a multiplicative factor of $|\alphabet|$, corresponding to $|\alphabet|$ calls to a prefix parser, where $\alphabet$ is the token alphabet. Importantly, next-token weights arise in many applications, including context-free language modeling \citep{jelinek-lafferty-1991-computation}, surprisal-based psycholinguistic modeling \citep{Hale}, and syntactically constrained generation from LLMs \citep{shin-etal-2021-constrained,poesia-etal-2022-synchromesh,loula2025syntactic}.\looseness=-1

\section{Preliminaries\texorpdfstring{\protect\footnote{For convenience, we provide a notation glossary in \cref{sec:notation-table}.}}{}}
\label{sec:background}

\paragraph{Basics.}
Let $\alphabet$ be an \defn{alphabet}, i.e., a finite set of symbols, and let $\strings$ denote the \defn{set of all strings} formed from the symbols of $\alphabet$, including the empty string $\emptystring$.
We write $\str \preceq \strt$ if $\str$ is a prefix of $\strt$, and $\str \strt$ to denote concatenation.
Let $\Tuple{\Weights, +, \cdot, 0, 1}$ be a commutative semiring, i.e., a set equipped with addition and multiplication satisfying the usual distributive and identity laws. For concreteness, the reader may take $\Weights = \Nonnegatives$.
A \defn{weighted language} is a function $\Lang\colon \strings \to \Weights$.
We call $\Lang$ a \defn{language model} when $\Weights = \Nonnegatives$ and $\sum_{\str \in \strings} \Lang(\str) = 1$.

\paragraph{Prefix languages.}
A central concept in this paper is the (weighted) \defn{prefix language} of $\Lang$:
\begin{align}
\label{eq:prefix-language}
\prefixLanguage{\Lang}(\str) \defeq  \sum_{\strt \in \strings} \indicator{\str \preceq \strt} \, \Lang\Parens{\strt}
\end{align}
In the language model setting where $\Weights = \Nonnegatives$, the weight $\Lang(\str)$ of any string $\str$ factorizes as a product of conditional prefix probabilities:
\begin{align}
\label{eq:factorization}
\Lang(\str) = \prefixLanguage{\Lang}( \eos \mid \str) \prod_{t=1}^{|\str|} \prefixLanguage{\Lang}( \chars_t \mid \str_{<t})
\end{align}
where $\str_{<t} \defeq \chars_1 {\cdots} \chars_{t-1}$, $\eos \notin \alphabet$ is a distinguished \defn{end-of-string symbol}, and the \defn{conditional next-token weights} are defined as follows:
\begin{align}
\prefixLanguage{\Lang}( \chars_{t} \mid \str_{<t}) &\defeq \begin{cases}
\frac{\Lang(\str_{<t})}{\prefixLanguage{\Lang}(\str_{<t})} &\textbf{if }\chars_{t}\!=\!\eos \\[4 pt]
\frac{\prefixLanguage{\Lang}(\str_{<t} \chars_{t})}{\prefixLanguage{\Lang}(\str_{<t} )} & \textbf{otherwise} \end{cases} \label{eq:conditional-next-token-weight}
\end{align}
provided that $\prefixLanguage{\Lang}(\str_{<t})>0$ and $\chars_t\!\in\!\alphabet\cup\{ \eos\}$. Note that $\prefixLanguage{\Lang}(\cdot \mid \str_{<t})$ is a distribution over $\alphabet \!\cup\! \{ \eos \}$.
Conditional prefix probabilities underlie context-free language modeling and surprisal-based psycholinguistic analysis. In \cref{sec:next-token-algorithms}, we develop efficient algorithms for computing the prefix weights of all single-symbol extensions simultaneously.

\label{sec:cfg-prelim}
A \defn{weighted context-free grammar}\protect\footnote{We assume the reader is familiar with WCFGs; for completeness, \cref{sec:wcfg-appendix} provides the necessary background.} (\defn{WCFG}, or simply \defn{CFG}) $\grammar$ is a tuple $\wcfgtuple$, where $\nonterms$ is a set of \defn{nonterminal} symbols, $\alphabet$ is an alphabet of \defn{terminal} symbols (disjoint from $\nonterms$), $\start \! \in \! \nonterms$ is a distinguished \defn{start symbol}, and $\rules$ is a finite set of weighted \defn{rules}.\footnote{\label{footnote:duplicate-rules}Without loss of generality, any two rules with the same left-hand side $\ntX$ and right-hand side $\valpha$, $\wproduction{\ntX}{\valpha}{\rw}$ and $\wproduction{\ntX}{\valpha}{\rw'}$, can be consolidated into the single rule $\wproduction{\ntX}{\valpha}{\rw + \rw'}$.\looseness=-1}
Each rule $\arule \in \rules$ is of the form $\wproduction{\ntX}{\valpha}{\rw}$ where $\ntX \in \nonterms$, $\valpha \in \kleene{(\alphabet \cup \nonterms)}$, and $\rw \in \Weights$.
The \defn{arity} of a rule $\arule$, denoted $\arity{\arule}$, is the length of its right-hand side. 
We say that a rule is \emph{nullary} if its arity is zero, \emph{unary} if its arity is one, \emph{binary} if its arity is two, and so on.
The \defn{size} of a grammar is $|\grammar|\eq\sum_{\arule \in \rules} 1+\arity{\arule}$. 
We say that a grammar is in \defn{canonical two-form} (\defn{CTF}) if each of its rules has a right-hand side of length at most two, and \defn{Chomsky normal form} (\defn{CNF}) if each of its rules takes one of the forms $\wproduction{\ntX}{\ntY \ntZ}{}$, $\wproduction{\ntX}{\sym}{}$, or $\wproduction{\start}{\emptystring}{}$ for $\ntX,\ntY,\ntZ \in \nonterms$, $\sym \in \alphabet$.\looseness=-1

A \defn{derivation tree} is a rooted, ordered tree in which each node is labeled with a symbol in $\nonterms \cup \alphabet$, each nonleaf node is connected to its children by a rule in $\rules$, and each leaf is either a terminal or a childless nonterminal (arising from a nullary rule).
The \defn{weight} of a derivation tree is the product of the weights of all rules applied in the tree.
Let $\Derivations{\gramsym}{\str}$ denote the set of derivation trees rooted at $\gramsym \in \nonterms \cup \alphabet$ whose \defn{yield} is $\str$. The \defn{weighted language} of $\gramsym$ is $\wl{\grammar}{\gramsym}(\str) \defeq \sum_{\tree \in \Derivations{\gramsym}{\str}} \weight{\tree}$, and the \defn{weighted language} of $\grammar$ is $\wl{\grammar}{}(\str) \defeq \wl{\grammar}{\start}(\str)$.
The \defn{total weight} of $\gramsym$ is $\totalWeight{\gramsym} \defeq \sum_{\str \in \strings} \wl{\grammar}{\gramsym}(\str)$.
A (tight) \defn{probabilistic context-free grammar} (\defn{PCFG}) is one where $\totalWeight{\gramsym} = 1$ for all $\gramsym \in \nonterms \cup \alphabet$.  If $\grammar$ is a PCFG, then $\wl{\grammar}{}$ is a language model \citep{booth1973,pcfgChi}.\looseness=-1

\section{Prefix Parsing is Just Parsing}
\label{sec:prefix-parsing}
This section introduces the \defn{prefix grammar transformation}, which allows us to turn any CFG parser into a \emph{prefix} parser simply by transforming its grammar into one that generates its prefix language.\footnote{Prefix grammars themselves are not novel.  Indeed, deriving one is an exercise on context-free grammars in \citeposs{bendavCS360Lecture10} lecture notes. However, we were unable to find a \emph{weighted} CFG version; much to our delight, that extension proved straightforward. In this light, our contribution is to demonstrate how the prefix grammar can be used to \emph{reduce} prefix parsing to ordinary parsing (with weights).}\textsuperscript{,}\footnote{Note that the prefix grammar can also be defined by \emph{composition} \citep{BarHillel61,pasti-etal-2023-intersection} of the input grammar and a specifically designed \emph{prefix transducer}; we discuss this view in detail in \cref{sec:composition}.}

\begin{definition}
\label{def:prefix-grammar-optimized}
\begin{subequations}
Given a grammar $\grammar=\wcfgtuple$, its \defn{prefix grammar} is
\[
\prefixGrammar \defeq \langle \nonterms \cup \{ \primeNT{\ntX} \}_{\ntX \in \nonterms } \cup \{ \preStart \}, \alphabet, \preStart, \rules \cup \primeRules \rangle
\]
where $\primeNT{\gramsym}$ denotes $\gramsym$ itself if $\gramsym \in \alphabet$ and a new \emph{prime} nonterminal if $\gramsym \in \nonterms$. Here $\preStart$ is a new start symbol, and the additional rules $\primeRules$ are defined below:\looseness=-1
\begin{flalign}
&\wproduction{\preStart}{\primeNT{\start}}{1} \label{eq:prefix-grammar-optimized-start} \\
&\wproduction{\preStart}{\emptystring}{\totalWeight{\start}}
\label{eq:prefix-grammar-optimized-exit}
\\
&\wproduction{\primeNT{\ntX}}{
\gramsym_1 \cdots \gramsym_{k-1} \, \primeNT{\gramsym_k}
}{\rw \cdot \totalWeight{\gramsym_{k+1}} {\cdots} \totalWeight{\gramsym_{K}} } \label{eq:prefix-grammar-optimized-borderline-exit}\\
&\quad\qquad\annotation{\text{for }\wproduction{\ntX}{\gramsym_1 {\cdots} \gramsym_K}{\rw} \in \rules,\, k \in 1, {\dots}, K}  \nonumber \end{flalign}
\end{subequations}
\end{definition}
\noindent The idea underlying the prefix grammar is that, at each level of the derivation tree, a prefix of a derived string must end inside the yield of exactly one rule application.
This is captured by \cref{eq:prefix-grammar-optimized-borderline-exit}: for each original rule $\wproduction{\ntX}{\gramsym_1 \cdots \gramsym_K}{\rw}$ and each border position $k$, the symbols $\gramsym_1, {\ldots}, \gramsym_{k-1}$ before the border are parsed normally, the symbol $\gramsym_k$ at the border is only partially matched (handled recursively via $\primeNT{\gramsym_k}$), and the symbols $\gramsym_{k+1}, {\ldots}, \gramsym_K$ after the border are summarized by their total weights $\totalWeight{\gramsym_{k+1}} \cdots \totalWeight{\gramsym_K}$.
Recall that $\totalWeight{\gramsym} = 1$ for PCFGs (\cref{sec:cfg-prelim}); for general WCFGs, the total weights can be computed as described in \cref{def:total-weight}.\looseness=-1

The following proposition establishes that \cref{def:prefix-grammar-optimized} correctly encodes the prefix language.\looseness=-1
\begin{restatable}{proposition}{restatablePrefixGrammarCorrectness}
Let $\grammar$ be a CFG, and $\prefixGrammar$ be its prefix grammar. Then, $\prefixGrammar$ correctly encodes the prefix language of $\grammar$ (i.e., $\wl{\prefixGrammar}{} = \prefixLanguage{\wl{\grammar}{}}$).
\label{thm:optimized-prefix-grammar-is-correct}    
\end{restatable}
\begin{proof}
See \cref{thm:optimized-prefix-grammar-is-correct-proof}.
\end{proof}
\noindent This provides a straightforward reduction of any ordinary parsing algorithm to a prefix-parsing one.  We say $\Parse$ is a \defn{correct parsing algorithm} if $\Parse(\grammar,\str) = \wl{\grammar}{}{\str}$ for all grammars $\grammar$ and strings $\str$. Although this framing treats $\Parse$ as a black box over arbitrary CFGs, most parsers in the literature are in fact defined only on a restricted class of grammars (e.g., CNF, nullary-free) and rely on an upstream \defn{normal-form conversion} $\phi$ (e.g., CNF conversion for CKY) that is typically treated implicitly; its size cost is frequently understated and propagates through our reduction---we track it explicitly in \cref{thm:prefix-parsing-runtime}. The following theorem shows that our reduction yields a correct prefix-parsing algorithm.
\begin{theorem}[Prefix parsing is just parsing]
Let $\grammar$ be a CFG over the alphabet $\alphabet$. Let $\prefixGrammar$ be its prefix grammar. If $\Parse$ is a correct parsing algorithm, then $\Parse(\prefixGrammar,\str) = \prefixLanguage{\wl{\grammar}{}}(\str)$ for all $\str \in \strings$.\looseness=-1
\end{theorem} 
\begin{proof}
The theorem immediately follows from \cref{thm:optimized-prefix-grammar-is-correct} and the premises.\looseness=-1
\end{proof}

In \cref{sec:app-prefix-parsers}, we describe specific instantiations of this approach with \IncrCKY and \Earley. The key to the efficiency of our reduction is that we can bound the size of the prefix grammar.
For efficiency, it is advisable to binarize the grammar (i.e., convert to canonical two-form) before applying the prefix grammar transformation. Although binarization triples the grammar size, it bounds the rule arity at two, which in turn keeps the prefix grammar within a small constant factor of the original:\looseness=-1

\begin{restatable}{proposition}{restatableGrammarSizeBound}
\label{prop:optimized-grammar-bound}
Let $\grammar$ be a context-free grammar in canonical two-form and let $\prefixGrammar$ be its prefix grammar (\cref{def:prefix-grammar-optimized}). Then, the size of $\prefixGrammar$ is bounded by
\begin{align}
|\prefixGrammar| \leq \frac{8}{3}|\grammar| +3
\end{align}
\end{restatable}
\begin{proof}
See \cref{prop:optimized-grammar-bound-proof}.
\end{proof}

\noindent \Cref{prop:optimized-grammar-bound} bounds $|\prefixGrammar|$ directly. However, as noted above, most parsers' runtime bounds involve a preprocessed grammar $\phi(\grammar)$ (e.g., CNF conversion for CKY) and often depend on structural parameters beyond total size, such as $|\nonterms|$, which grow differently under $\phi$ (\cref{sec:grammar-transformations}). The following theorem lifts \cref{prop:optimized-grammar-bound} to parser runtime by allowing the runtime to be any polynomial in string length and a vector $\boldsymbol{\nu}$ of grammar-structural parameters.

\begin{restatable}{theorem}{restatablePrefixParsingRuntime} 
\label{thm:prefix-parsing-runtime}
Let $\Parse$ be a parsing algorithm whose runtime after grammar preprocessing by $\phi$ satisfies\looseness=-1
\[ \rt{\Parse(\grammar, \str)} = \bigo{\textstyle\sum_{i=1}^m g_i(\boldsymbol{\nu}(\phi(\grammar))) \cdot h_i(|\str|)} \]
where $\boldsymbol{\nu}$ maps a grammar $\grammar' = \Tuple{\nonterms', \alphabet', \start', \rules'}$ to a vector of its structural parameters (e.g., $\boldsymbol{\nu}(\grammar') = [|\grammar'|, |\nonterms'|, |\rules'|, |\alphabet'|]^\top$), and each $g_i$ and $h_i$ is a polynomial.\footnote{We assume the one-time cost of $\phi$ is amortized over all strings parsed with the same grammar.} Then,
\[ \rt{\Parse(\smash[t]{\prefixGrammar}, \str)} = \bigo{\textstyle\sum_{i=1}^m g_i(\boldsymbol{\nu}(\phi(\smash[t]{\prefixGrammar}))) \cdot h_i(|\str|)} \]
In particular, the $|\str|$-dependence of each term is unchanged; the entire overhead of prefix parsing is captured by replacing $\boldsymbol{\nu}(\phi(\grammar))$ with $\boldsymbol{\nu}(\phi(\prefixGrammar))$.
\end{restatable}
\begin{proof}
Apply $\Parse$'s runtime guarantee with input $\prefixGrammar$ in place of $\grammar$.
\end{proof}

\noindent Observe that \cref{thm:prefix-parsing-runtime} makes no claim about the magnitude of $\rt{\Parse(\grammar, \str)}$ itself; that is a property of the parser. For example, \IncrCKY runs in $\bigo{|\grammar| \cdot |\str|^3}$ on CNF grammars, which fits the theorem with $\phi$ set to CNF conversion. Because the unary-removal step of $\phi$ can inflate grammar size by a factor of $|\nonterms|$ (\cref{sec:grammar-transformations}), applying \IncrCKY to $\prefixGrammar$ runs in $\bigo{|\phi(\smash[t]{\prefixGrammar})| \cdot |\str|^3}$---a factor of $|\nonterms|$ looser than the bound for $\grammar$ (\cref{sec:prefix-cky}). Earley's lighter preprocessing incurs less blowup (\cref{sec:prefix-earley}).

\section{Next-Token Weight Vector Algorithms}\label{sec:next-token-algorithms}
\noindent Given a string $\str= \chars_1 {\cdots} \chars_N$ and a CFG $\grammar$, we now study how to efficiently compute the \defn{next-token weight vector} $\nextTokenWeights{\sym}{\str} = \prefixLanguage{\wl{\grammar}{}}(\str\sym)$, whose components give the prefix weight of each one-symbol extension $\str\sym$ for $\sym \in \alphabet$. This vector is the key ingredient for the incremental processing applications discussed in \cref{sec:introduction}, including constrained generation with a PCFG.
Importantly, by \cref{thm:optimized-prefix-grammar-is-correct}, computing the next-token weight vector reduces to ordinary parsing: $\nextTokenWeights{\sym}{\str} = \wl{\prefixGrammar}{}{\str\sym}$. Yet a na\"ive implementation would still require $|\alphabet|$ parser calls, one per $\sym \in \alphabet$, i.e., $\bigo{|\alphabet| \cdot \rt{\Parse(\smash[t]{\prefixGrammar}, \str)}}$ time per next-token vector $\nextTokenWeights{}{\str}$.

A less na\"ive approach is to use an \defn{incremental parsing algorithm}---one whose evaluation on a prefix $\chars_1 {\cdots} \chars_{N-1}$ leaves enough cached information to extend to $\chars_1 {\cdots} \chars_N$ cheaply.
We refer to this cached information as the \defn{prefix state} and assume it is managed by memoization, so that the same prefix is parsed at most once. The precise representation of the prefix state is parser-specific; for incremental CKY (\cref{alg:incremental-cky} in \cref{sec:app-prefix-parsers}) it is the portion of the parse chart filled in for $\chars_1 {\cdots} \chars_{N-1}$. Non-incremental CKY on $\prefixGrammar$ runs in $\bigo{|\phi(\smash[t]{\prefixGrammar})|N^3}$; incremental CKY amortizes this cost across prefixes, extending the parse by one token in only $\bigo{|\phi(\smash[t]{\prefixGrammar})|N^2}$. Immutability lets us safely reuse the prefix state across all $|\alphabet|$ candidate one-token extensions of $\chars_1 {\cdots} \chars_{N-1}$, as the cached entries from the prefix do not depend on the choice of next token. This reduces the baseline to $\bigo{|\phi(\smash[t]{\prefixGrammar})|N^2 \cdot |\alphabet|}$ per token---an improvement by a factor of $N$---but the $|\alphabet|$ factor remains.\looseness=-1

In contrast, we propose an approach that eliminates the $|\alphabet|$ factor entirely, computing the full vector $\nextTokenWeights{}{\str}$ in $\bigo{\rt{\Parse(\smash[t]{\prefixGrammar}, \str)}}$ time per next-token vector---the same asymptotic cost as a single parser call---by combining two key ingredients, \emph{lattice parsing} and \emph{algorithmic differentiation}, which we introduce next. Our approach applies to any parser that supports lattice input and whose evaluation consists of semiring operations---mild requirements satisfied by standard algorithms, such as incremental CKY and Earley (see \cref{sec:app-prefix-parsers}). Moreover, when the base parser is incremental, our next-token algorithm inherits its per-prefix amortization.

\paragraph{An algorithmic differentiation reduction.}
We start by defining the aggregation function:
\begin{align}
\partition{\str}(\parameters) \defeq \sum_{\sym \in \alphabet} \nextTokenWeights{\sym}{\str} \cdot \param_{\sym}
\label{eq:partition-parameter}
\end{align}
where $\parameters \in \Weights^{\alphabet}$ is a vector of formal variables, introduced solely to enable differentiation. 
Because this sum is linear in $\parameters$, its gradient immediately yields the next-token weight vector:\footnote{The gradient here is \emph{algebraic}: $\nabla[{\color{ValueColor}x} + {\color{ValueColor}y}] = \nabla[{\color{ValueColor}x}] + \nabla[{\color{ValueColor}y}]$ and $\nabla[{\color{ValueColor}x} \cdot {\color{ValueColor}y}] = \nabla[{\color{ValueColor}x}] \cdot {\color{ValueColor}y} + {\color{ValueColor}x} \cdot \nabla[{\color{ValueColor}y}]$, requiring only semiring operations \citep[see, e.g.,][]{omega-continuous}.}
\begin{align}
\nabla_{\parameters} \partition{\str}(\parameters) = \nextTokenWeights{}{\str} \qquad \forall \parameters \in \Weights^{\alphabet}
\label{eq:partition-parameter-derived}
\end{align}
This reduction is productive because there is an efficient method for computing $\partition{\str}(\parameters)$---namely, \emph{lattice parsing}, which replaces the input string with a compact automaton encoding all $|\alphabet|$ one-symbol extensions at once.
Moreover, thanks to algorithmic differentiation, we have the following \emph{meta-theorem} that ensures the computation of the gradient
$\nabla_{\parameters} \partition{\str}(\parameters)$ is equally efficient.\looseness=-1
\begin{theorem}[Next-token meta-theorem]
\label{thm:meta-theorem}
For any algorithm
$\LatticeParser$ that correctly computes $\partition{\str}(\parameters)$ and whose evaluation is a sequence of semiring operations,\footnote{Control flow (for-loops, if-statements) is permitted provided it does not depend on $\parameters$---this ensures the computed function is a polynomial in $\parameters$, so symbolic differentiation applies.} we can use reverse-mode algorithmic differentiation to derive the gradient algorithm $\Backwards{\LatticeParser}$ that correctly computes $\nabla_\parameters \partition{\str}(\parameters) = \nextTokenWeights{}{\str}$ using at most $4$ times as many arithmetic operations as $\LatticeParser$.\looseness=-1
\end{theorem}
\begin{proof}
The proof follows from \cref{eq:partition-parameter-derived} and the cheap gradient principle \citep[Eq.\ 3.14]{griewank}, which bounds the gradient computation at $4\times$ the original function's operation count. Since this bound counts arithmetic operations irrespective of their interpretation, it carries over from $\R$ to any commutative semiring $\Weights$.
\end{proof}

\paragraph{Lattice parsing.}
\label{sec:lattice-parsing}
Next, we show how to efficiently compute $\partition{\str}(\parameters)$ via lattice parsing. A \defn{word lattice} \citep{hall2005best} is an acyclic weighted finite-state automaton (WFSA; see \cref{sec:wfsa-appendix}) representing a weighted set of strings. A \defn{lattice parser} $\LatticeParser$ is an algorithm that, given a grammar $\grammar$ and a word lattice $\lattice$, computes the following product:\footnote{Lattice parsers are routine extensions of ordinary parsers \citep[\S3.2]{hall2005best}; ordinary parsing is the special case where the lattice contains the single string $\str$ with weight one.\looseness=-1
}
\begin{align}
\LatticeParser(\grammar,\lattice) = \smashoperator{\sum_{\str \in \strings}} \wl{\lattice}{}{\str} \!\cdot\! \wl{\grammar}{}{\str}\!
\label{eq:lattice-parser}
\end{align}
That is, the lattice parser computes the inner product of the lattice's and grammar's weighted languages. Efficient lattice parsers can be derived by minimally modifying existing parsing algorithms like Earley and CKY \citep{hall2005best}.\looseness=-1

Our approach to computing the \emph{aggregation function} $\partition{\str}(\parameters)$ is based on parsing the next-token lattice defined below.
\begin{definition}
\label{def:lattice}
Let $\str\in \kleene{\alphabet}$  and $\parameters \in \Weights^{\alphabet}$. The \defn{next-token lattice} $\lattice_{\str}(\parameters)$ is the WFSA below:
\begin{center}
\begin{tikzpicture}[>=stealth,auto] 
\tikzset{
  initial text={},
  every edge/.style={font=\scriptsize,draw},
  every state/.style={
    rectangle,
    rounded corners=4pt,
    draw,
    minimum height=1.em,
    minimum width=1.em,
    inner sep=3pt,
    align=center,
    font=\scriptsize
  }  
}
\node[state, initial] (s1) {$0$};
\node (dots) [right=of s1] {$\cdots$};
\node[state] (sn) [right=of dots]{$N$};
\node[state, accepting] (sn+) [right=2.1 cm of sn] {$N{+}1$};
\path[->] 
(sn) edge node {
$\{\sym /\param_{\sym}\}_{\sym \in \alphabet}$
} (sn+)
(s1) edge node {$\str_1/1$}  (dots)
(dots) edge  node {$\str_N/1$}  (sn)
;
\end{tikzpicture}
\end{center}
\noindent where the edge labeled $\{\sym /\param_{\sym}\}_{\sym \in \alphabet}$ is shorthand for a set of transitions---one for each symbol $\sym \in \alphabet$.
\end{definition}
\noindent Since the lattice accepts exactly the strings $\str\sym$ for $\sym \in \alphabet$, each weighted by $\param_{\sym}$, parsing it with the prefix grammar recovers the aggregation function:
\begin{align}
\LatticeParser(\prefixGrammar,\nextTokenLattice{\str})= \partition{\str}(\parameters)    
\end{align}

\paragraph{Putting it together.}
We now combine the three ingredients---the prefix grammar (\cref{def:prefix-grammar-optimized}), the next-token lattice (\cref{def:lattice}), and the next-token meta-theorem (\cref{thm:meta-theorem})---into a complete next-token weight vector algorithm.
Let $\grammar$ be a CFG with prefix grammar $\prefixGrammar$, and let $\str = \chars_1 {\cdots} \chars_N$ be a string. Suppose that $\LatticeParser$ is an algorithm whose evaluation consists of a sequence of semiring operations. Then, by the next-token meta-theorem (\cref{thm:meta-theorem}), there exists a gradient algorithm $\Backwards{\LatticeParser}$ that computes the next-token weight vector $\nextTokenWeights{}{\str} \!=\! \nabla_{\parameters}\partition{\str}(\parameters)$ with the same asymptotic complexity as one evaluation of $\LatticeParser(\prefixGrammar,\nextTokenLattice{\str})$.
We present two concrete instantiations of this framework in the appendix: one based on CKY (\cref{sec:app-next-token-cky}) and one based on Earley's algorithm (\cref{sec:app-next-token-earley}).
The empirical results displayed in \cref{fig:runtime-comparison-three-algorithms}---which use the Earley-based instantiation---confirm that computing $\nextTokenWeights{}{\str}$ is slower than prefix parsing by only a small constant factor.
Finally, in many applications, we also need the total weight of a string $\wl{\grammar}{}{\str}$ alongside its prefix weight.\footnote{Define $\grammarEos \defeq \Tuple{ \nonterms, \alphabet \cup \{ \eos \}, \start', \rules \cup \{ \wproduction{\start'}{\start \, \eos}{1} \} }$ where $\start' \notin \nonterms$. Since $\wl{\grammarEos}{}{\str\, \eos} = \wl{\grammar}{}{\str}$, applying the next-token weight vector algorithm to the prefix grammar $\prefixLanguage{\grammarEos}$ yields $\wl{\grammar}{}{\str}$ as the $\eos$ component of the vector, alongside all other next-token weights. The transformation is efficient: it adds a single rule, so $|\grammarEos| = |\grammar| + 3$.}\looseness=-1

\section*{Conclusion}
We showed that prefix parsing is just ordinary parsing---applied to the prefix grammar, a compact CFG encoding of the prefix language. Given any CFG, we construct a new CFG whose ordinary parser computes prefix weights of the original. Unlike previous approaches that extend existing parsing algorithms with complex modifications, our reduction provides a unified recipe for prefix parsing.
Additionally, using algorithmic differentiation, we provide the fastest known algorithm for computing the next-token weight vector.
With our reduction in hand, there is no longer a need to develop bespoke prefix-parsing algorithms. A natural direction for future work is to develop a generalized notion of a \emph{prefix transformation} that extends to other formalisms, such as context-sensitive grammars.\looseness=-1

\section*{Limitations}
The limitations of our approach are discussed in context throughout the main text; however, we would like to remind that one of the main limitations of our work is that, in the general case, \cref{thm:prefix-parsing-runtime} does not guarantee that a prefix parser has the same asymptotic runtime as the underlying base parser. Additionally, we would like to recall that our prefix parser works with any parser of choice, however we only ran experiments with CKY and Earley's.

\section*{Acknowledgments}
We thank \href{https://www.cs.jhu.edu/~jason/}{Jason Eisner} and \href{https://franznowak.github.io/}{Franz Nowak} for useful discussions. 
We also thank \href{https://joaoloula.github.io/}{Jo\~{a}o Loula}, \href{https://benlebrun.github.io/about/}{Ben LeBrun}, and \href{https://github.com/leo-du}{Leo Du} for early adoption of our method \citep[see][]{loula2025syntactic}. We used generative AI to improve our writing and to help debug the code. Every modification introduced by generative AI was carefully reviewed by the authors, which take full responsibility for it.

\bibliography{main}

\appendix

\onecolumn

\addcontentsline{toc}{section}{Appendix}

\begingroup
\addtocontents{toc}{\protect\setcounter{tocdepth}{2}}
\setcounter{tocdepth}{2}  \renewcommand{\contentsname}{Appendix Contents}
\tableofcontents
\endgroup

\clearpage
\section{Notation Glossary}
\label{sec:notation-table}
\begin{table}[h]\footnotesize
\centering
\begin{tabular}{@{}ll@{}}
\hline
\textbf{Symbol} & \textbf{Description} \\
\hline
\multicolumn{2}{@{}l@{}}{\textit{Strings and languages (\cref{sec:background})}} \\
$\alphabet$ & alphabet (finite set of terminal symbols) \\
$\strings$ & set of all strings over $\alphabet$ \\
$\emptystring \in \strings$ & empty string \\
$\str, \strt \in \strings$ & strings \\
$\chars, \sym \in \alphabet$ & terminal symbols \\
$\kleeneplus{\alphabet}$ & set of nonempty strings over $\alphabet$ \\
$\preceq$ & prefix relation ($\str \preceq \strt$ iff $\str$ is a prefix of $\strt$) \\
$\eos \notin \alphabet$ & end-of-string symbol \\
$\Tuple{\Weights, +, \cdot, 0, 1}$ & commutative semiring of weights \\
$\Lang \colon \strings \to \Weights$ & weighted language \\
$\prefixLanguage{\Lang} \colon \strings \to \Weights$ & prefix language of $\Lang$ \\
\hline
\multicolumn{2}{@{}l@{}}{\textit{Grammars (\cref{sec:cfg-prelim}, \cref{sec:wcfg-appendix})}} \\
$\grammar = \wcfgtuple$ & weighted context-free grammar (WCFG) \\
$\nonterms$ & set of nonterminal symbols \\
$\start \in \nonterms$ & start symbol \\
$\rules$ & finite set of weighted rules \\
$\gramsym \in \nonterms \cup \alphabet$ & grammar symbol (nonterminal or terminal) \\
$\valpha \in \kleene{(\nonterms \cup \alphabet)}$ & rule right-hand side \\
$\arule, (\wproduction{\ntX}{\valpha}{\rw}) \in \rules$ & rule; $\ntX$ rewrites to $\valpha$ with weight $\rw \in \Weights$ \\
$|\grammar| \in \mathbb{N}$ & grammar size \\
$\arity{\arule} \in \mathbb{N}$ & arity of rule $\arule$ \\
$\tree \in \Derivations$ & derivation tree of $\grammar$ \\
$\tree \in \Derivations{\gramsym}\subseteq \Derivations{}$ & $\gramsym$-rooted derivation tree \\
$\Derivations{\gramsym}{\str} \subseteq \Derivations{\gramsym}$ & set of derivation trees rooted at $\gramsym$ with yield $\str$ \\
$\yield{\tree} \in \strings$ & yield of derivation tree $\tree$ \\
$\weight{\tree} \in \Weights$ & weight of derivation tree $\tree$ \\
$\wl{\grammar}{}{\str} \in \Weights$ & weighted language of $\grammar$: total weight of $\str$ \\
$\totalWeight{\gramsym} \in \Weights$ & total weight of symbol $\gramsym \in \alphabet \cup \nonterms$ \\
\hline
\multicolumn{2}{@{}l@{}}{\textit{Prefix grammar (\cref{sec:prefix-parsing})}} \\
$\prefixGrammar$ & prefix grammar of $\grammar$ \\
$\primeNT{\ntX}$ & prime nonterminal (partially matched $\ntX \in \nonterms$, $\primeNT{\ntX} \notin \nonterms$) \\
$\preStart$ & prefix start symbol ($\preStart \ne \start$, $\preStart \notin \nonterms$) \\
\hline
\multicolumn{2}{@{}l@{}}{\textit{Parsers and preprocessing (\cref{sec:prefix-parsing}, \cref{sec:lattice-parsing})}} \\
$\phi$ & grammar preprocessing function (e.g., $\EnsureCNF$; see \cref{sec:grammar-transformations}) \\
$\Parse$ & abstract parsing algorithm \\
$\LatticeParser$ & abstract lattice parser \\
\hline
\multicolumn{2}{@{}l@{}}{\textit{Next-token algorithms (\cref{sec:next-token-algorithms})}} \\
$\nextTokenWeights{}{\str} \in \Weights^{\alphabet}$ & next-token weight vector \\
$\nextTokenWeights{\sym}{\str} \in \Weights$ & next-token weight for terminal $\sym$ \\
$\parameters \in \Weights^{\alphabet}$ & formal parameter vector \\
$\param_{\sym} \in \Weights$ & parameter for terminal $\sym$ \\
$\partition{\str}(\parameters) \in \Weights$ & aggregation function \\
$\lattice$ & word lattice (acyclic WFSA) \\
$\nextTokenLattice{\str}$ & next-token lattice \\
$\z(i, \ntX) \in \Weights$ & forward value \\
$\dz(i, \ntX) \in \Weights$ & backward value \\
$\inside{}{k}{i,\ntX} \in \Weights$ & inside weight \\
\hline
\end{tabular}
\end{table}

\clearpage

\clearpage
\section{Additional Experimental Results}
\label{sec:EarleyX}
\begin{figure*}[h]
\centering
\begin{subfigure}[b]{0.49\linewidth}
    \includegraphics[width=\linewidth]{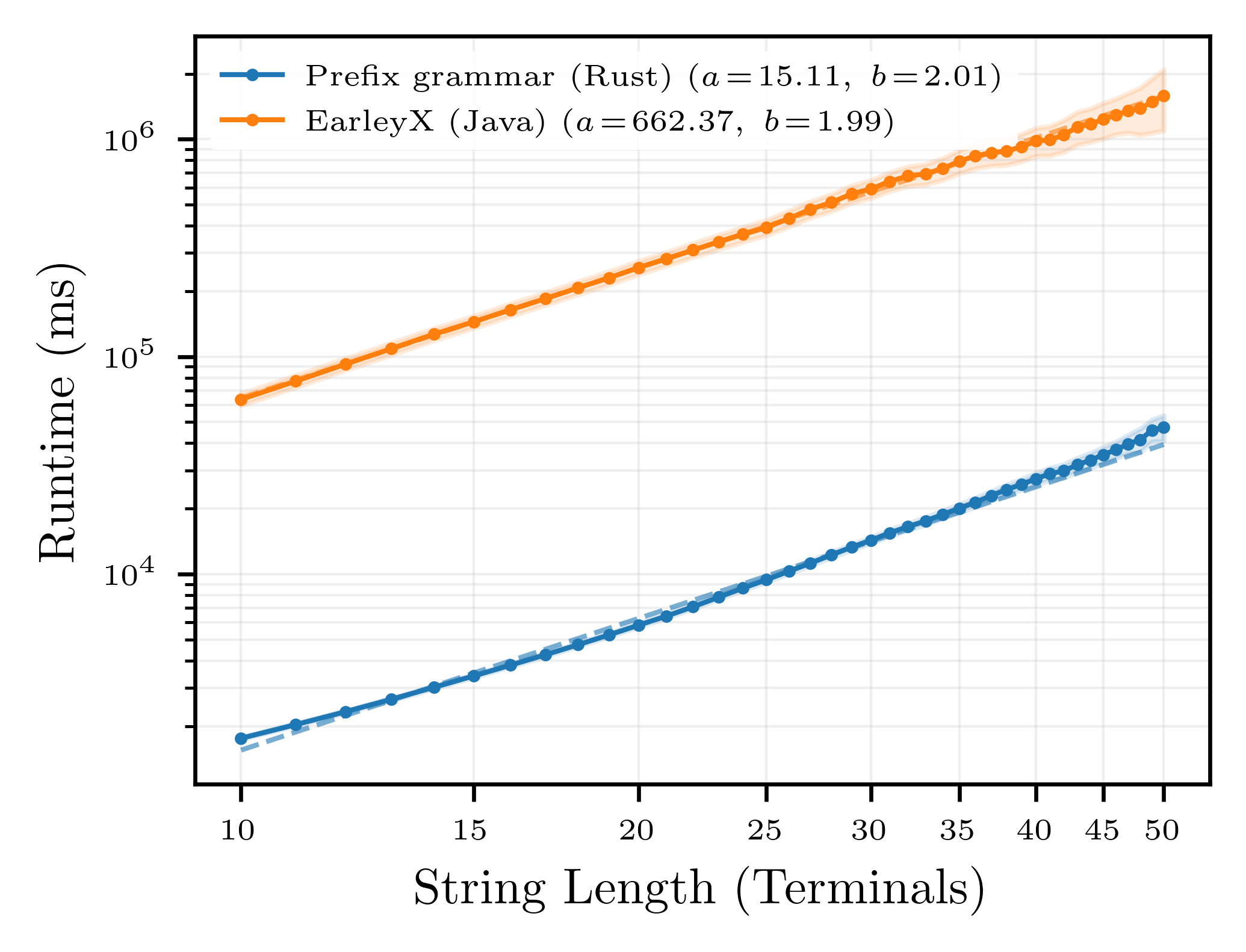}
    \label{fig:runtime-comparison}
\end{subfigure}
\begin{subfigure}[b]{0.49\linewidth}
    \includegraphics[width=\linewidth]{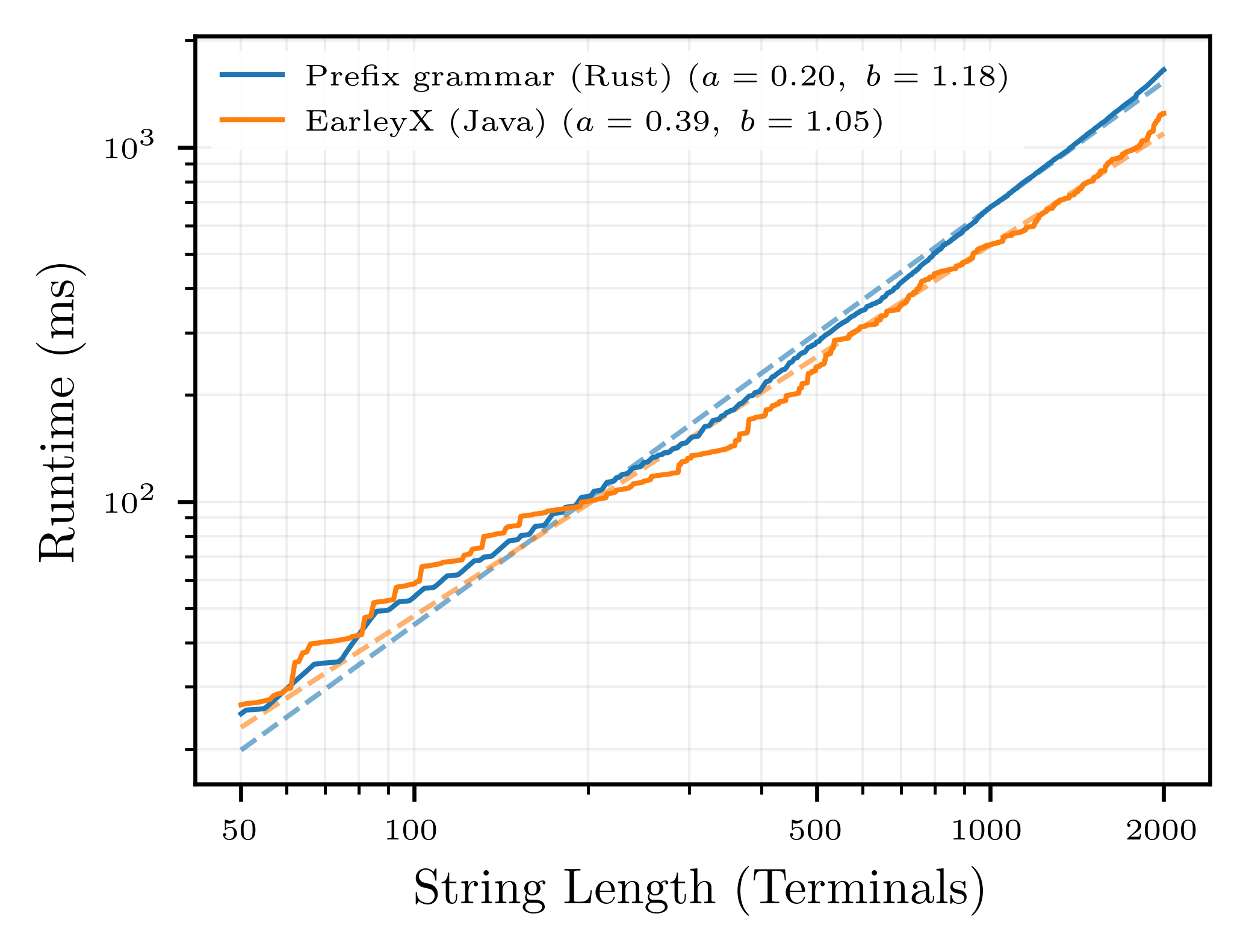}
    \label{fig:earleyx-comparison}
\end{subfigure}
\vspace{-0.5 cm}
\caption{\footnotesize
\textbf{Left: Our prefix parser vs.\ EarleyX {\normalfont\citep{luong2013-parsing-discourse}}.} We compare the two implementations on WSJ 5000 and 500 sentences. Our algorithm uses Earley's with the prefix grammar (\cref{alg:earley-prefix-parser}), while EarleyX uses \citeposs{stolcke-1995-efficient} algorithm to compute prefix probabilities (itself an adaptation of Earley's). As in \cref{fig:runtime-comparison-three-algorithms}, we fit a power law $r(N) = aN^b$ via least-squares regression on the log--log scale. The fitted slopes are nearly identical ($2.01$ vs.\ $1.99$), while the fitted coefficient is substantially better in our implementation. This is consistent with the fact that Stolcke's algorithm pays an additional multiplicative factor of $|\rules|$, which our Earley implementation avoids via an optimization at the completion step (\cref{alg:earley}).
\textbf{Right: Sparse grammar comparison.}
We repeat the experiment on the \emph{Social Discourse} grammar \citep[\texttt{EarleyX}]{luong2013-parsing-discourse}, a sparse grammar with 72{,}712 rules and 233 nonterminals. The graph reports the runtime for each prefix of a sequence of concatenated sentences (\texttt{socialall.ortho.yld.concat}). Both parsers achieve near-linear time on this sparse grammar (fitted exponents $1.05$ for EarleyX and $1.18$ for our implementation). For this experiment, we used the specialized \texttt{EarleyParserSparse} from \citet{luong2013-parsing-discourse}, which is highly optimized for sparse grammars.
}
\looseness=-1
\vspace{-5 mm}
 \label{fig:earleyx-comparison-two-grammars}
\end{figure*}

\clearpage
\section{Additional Background}
\subsection{Weighted Context-Free Grammars}
\label{sec:wcfg-appendix}

Let $\grammar = \wcfgtuple$ be a CFG. We develop the weighted language $\wl{\grammar}{}$ (and a few other useful concepts) by first defining (weighted) derivation trees, then summing over trees with a given yield. A \defn{derivation tree} $\tree$ is a rooted, ordered tree satisfying: (i) each node is labeled with a symbol in $\nonterms \cup \alphabet$, (ii) each nonleaf node is connected to its children by a rule in $\rules$, and (iii) each leaf is either a terminal or a childless nonterminal. The \defn{set of all derivation trees} is defined recursively as the smallest set $\Derivations$ satisfying:
\newcommand{\MyTree}[2][]{\tikz[CenterFix,level distance=.6cm,#1]{\Tree [.{\ensuremath{#2}} \edge[roof,fill=black!25]; {\hphantom{ab}} ]}}
\newcommand{\generictreetop}[0]{\raisebox{0pt}[1em][1.5em]{\tikz[CenterFix, scale=.7, sibling distance=.8cm, level distance=1cm, every node/.style={font=\footnotesize}]{
\node (xK) {\ntX}
  child { node[yshift=-.8em] { 
  \MyTree[scale=.9]{\ensuremath{\gramsym_1}} } 
  }
  child { node { \ensuremath{\cdots} } edge from parent[draw=none] }
  child { node[yshift=-.8em] { 
  \MyTree[scale=.9]{\ensuremath{\gramsym_K}} } 
  }
;
}}}
\newcommand{\genericsubtree}[2][]{\raisebox{0pt}[1em][1em]{\MyTree[font=\footnotesize,#1]{\ensuremath{\gramsym_{#2}}}}}
\begin{align}
\Derivations \defeq \ensuremath{\alphabet} \,\cup \label{def:derivations} \left\{\!
\generictreetop
\!\middle| \;
\wproduction{\ntX}{\gramsym_1\!\cdots\!\gramsym_K}{} \in \rules, \MyTree[scale=.9]{\gramsym_{1}}, {\ldots}, \MyTree[scale=.9]{\gramsym_{K}}  \in \Derivations
\right\}\nonumber
\end{align}
Note that the subtrees rooted at the internal nodes of $\tree$ are themselves derivation trees.
Let $\rootnode{\tree}$ denote the root label of \tree, and we write $\Derivations{\gramsym} \defeq \{ \tree \in \Derivations \colon \rootnode{\tree} = \gramsym \}$ to denote the derivations with root $\gramsym \in \alphabet \cup \nonterms$.
The \defn{yield} $\yield{\tree}$ of a tree $\tree$ is recursively defined as
\begin{subequations}
\begin{flalign}
&\bullet \textbf{if } \tree \in \alphabet
\text{: }\yield{\tree} \defeq \tree  &&\\
&\bullet\textbf{otherwise: }
\yield{\!
\generictreetop
\!}
\defeq
\yield{\genericsubtree{1}} \cdots
\yield{\genericsubtree{K}} &&
\end{flalign}
\end{subequations}
Note that the yield $\yield{\tree}$ of a childless node is $\emptystring$. A \defn{derivation forest} is a set of derivation trees; its weight is the sum of the weights of its trees. We write
$\Derivations{}{\str}\defeq \{ \tree \in \Derivations \colon \yield{\tree} = \str\}$ for the derivation forest with fixed yield $\str$, and $\Derivations{\gramsym}{\str} \defeq \Derivations{\gramsym} \cap \Derivations{}{\str}$ for the forest with root $\gramsym$ and yield $\str$.
The \defn{weight} $\weight{\tree}$ of the derivation tree \tree is defined as\looseness=-1
\begin{subequations}
\label{eq:tree-weight}
\begin{flalign}
&\bullet \textbf{if } \tree \in \alphabet
\text{: }\weight{\tree} \defeq 1  &&\\
&\bullet\textbf{otherwise: }
\weight{\!\generictreetop\!}
\defeq
\weight{\ntX \to \gramsym_1 \cdots \gramsym_K} \cdot
\weight{\genericsubtree{1}}\!\cdots\!\weight{\genericsubtree{K}} &&
\end{flalign}
\end{subequations}
\noindent For every $\gramsym \in \alphabet \cup \nonterms$, we define its \defn{weighted language} as
\begin{align}
\wl{\grammar}{\gramsym}{\str}
&\defeq
\sum_{\tree \in \Derivations{\gramsym}{\str}}
\!\!\!  \weight{\tree} 
\label{eq:weighted-language}
\end{align}
Every CFG defines a weighted language $\wl{\grammar}{}$ by taking the language of the root symbol:
\begin{align}
    \wl{\grammar}{}{\str} \defeq \wl{\grammar}{\start}{\str}
\end{align}
\noindent For each symbol $\gramsym \in \alphabet \cup \nonterms$, we define $\totalWeight{\gramsym}$ as the \defn{total weight}\label{def:total-weight} of all $\gramsym$-rooted derivations:\looseness=-1
\begin{align}
\totalWeight{\gramsym} = \sum_{\tree \in \Derivations{\gramsym}} \weight{\tree}
\end{align}
The total weights satisfy the following system of polynomial equations:
\begin{align}
\totalWeight{\gramsym} =
\begin{cases}
1 & \text{for }\gramsym \in \alphabet \\
\sum_{(\wproduction{\gramsym}{\gramsym_1 \cdots \gramsym_M}{\rw}) \in \rules} \rw \cdot \totalWeight{\gramsym_1} \cdots \totalWeight{\gramsym_M} & \text{for }\gramsym \in \nonterms
\end{cases}
\label{eq:total-weight-system}
\end{align}
The total weights are well-defined in any $\omega$-continuous semiring \citep{Handbook-Automata}; examples include the nonnegative extended reals. They can be approximated, e.g., by fixed-point iteration or generalized Newton's method \citep{Nederhof2008,esparza2007newton}.

\subsection{Grammar Transformations}
\label{sec:grammar-transformations}
A grammar transformation is a function $\phi$ that maps a CFG $\grammar$ to another CFG $\phi(\grammar)$. We say that $\phi$ is \emph{semantics preserving} if $\wl{\grammar}{}=\wl{\phi(\grammar)}{}$. Many transformations used to preprocess grammars are semantics preserving, including \defn{dead-rule elimination}, \defn{binarization}, \defn{nullary removal}, \defn{unary removal}, and \defn{terminal separation}. The purpose of most of these transformations is self-explanatory to a reader familiar with CFGs; we refer the reader to \citet{cotterell2022lecture7} for detailed descriptions. All transformations discussed in this section are included in our \href{https://github.com/genlm/prefix-acl-26}{code release}. Converting a grammar into CNF is itself a (composed) transformation $\EnsureCNF$: apply \EnsureCTF, then \EnsureNullaryFree, then \EnsureUnaryFree, then \EnsureTerminalSep.\footnote{Note that the order in which these transformations are applied matters. For example, applying \EnsureNullaryFree before \EnsureCTF can cause exponential blowup, as a rule of arity $K$ may spawn up to $2^K$ variants \citep{LangeLeiss2009}. Placing terminal separation last prevents subsequent multiplicative transformations from amplifying its additive cost.} In general, a grammar transformation $\phi$ can change both the size $|\phi(\grammar)|$ of the transformed grammar and the number of nonterminals $|\nonterms_\phi|$. We bound both for the transformations above:
\begin{itemize}
\item Dead rule elimination:
\begin{align} |\phi(\grammar)| &\le |\grammar|, & |\nonterms_\phi| &\le |\nonterms| \label{eq:bound-dead-rule} \end{align}
\item \EnsureCTF (binarization):
\begin{align} |\phi(\grammar)| &\le 3 |\grammar|, & |\nonterms_\phi| &\le |\nonterms| + |\grammar| - |\rules| \label{eq:bound-binarize} \end{align}
The $|\grammar| - |\rules|$ term accounts for intermediate nonterminals introduced when splitting rules of arity $> 2$.
\item \EnsureNullaryFree (nullary removal), provided $\grammar$ is in CTF:
\begin{align} |\phi(\grammar)| &\le \tfrac{7}{3} |\grammar| + 3, & |\nonterms_\phi| &\le |\nonterms| + 1 \label{eq:bound-nullary} \end{align}
\item \EnsureUnaryFree (unary removal):
\begin{align} |\phi(\grammar)| &\le |\nonterms| \cdot |\grammar|, & |\nonterms_\phi| &\le |\nonterms| \label{eq:bound-unary} \end{align}
\item \EnsureUnaryCycleFree (unary cycle removal) from \citet{opedal-etal-2023-efficient}:
\begin{align} |\phi(\grammar)| &\le |\grammar| + 2|\nonterms|^2, & |\nonterms_\phi| &\le 2|\nonterms| \label{eq:bound-unary-cycle} \end{align}
The $|\nonterms|^2$ term bounds the closure of the unary-rule graph; it is tight in the worst case (a single strongly connected component) and often far smaller in practice.
\item \EnsureTerminalSep (terminal separation):
\begin{align} |\phi(\grammar)| &\le |\grammar| + 2 |\alphabet|, & |\nonterms_\phi| &\le |\nonterms| + |\alphabet| \label{eq:bound-terminal-sep} \end{align}
\item \EnsureCNF (composing the above in the order \EnsureCTF, \EnsureNullaryFree, \EnsureUnaryFree, \EnsureTerminalSep). Tracing the pipeline:
\begin{enumerate}
\item $\EnsureCTF$: $|G_1| \le 3|\grammar|$,\; $|N_1| \le |\nonterms| + |\grammar| - |\rules|$ \hfill (\cref{eq:bound-binarize})
\item $\EnsureNullaryFree$: $|G_2| \le \tfrac{7}{3}|G_1| + 3 \le 7|\grammar| + 3$,\; $|N_2| \le |N_1| + 1$ \hfill (\cref{eq:bound-nullary})
\item $\EnsureUnaryFree$: $|G_3| \le |N_2| \cdot |G_2| \le (|\nonterms| + |\grammar| - |\rules| + 1)(7|\grammar| + 3)$,\; $|N_3| \le |N_2|$ \hfill (\cref{eq:bound-unary})
\item $\EnsureTerminalSep$: $|G_4| \le |G_3| + 2|\alphabet|$,\; $|N_4| \le |N_3| + |\alphabet|$ \hfill (\cref{eq:bound-terminal-sep})
\end{enumerate}
Summing yields:
\begin{align}
\begin{split}
|\phi(\grammar)| &\le (|\nonterms| + |\grammar| - |\rules| + 1)(7|\grammar| + 3) + 2|\alphabet| \\
|\nonterms_\phi| &\le |\nonterms| + |\grammar| - |\rules| + 1 + |\alphabet|
\end{split}
\label{eq:bound-cnf}
\end{align}
The dominant factor is unary removal; the $|\grammar|$ term in the first factor arises because binarization introduces up to $|\grammar| - |\rules|$ intermediate nonterminals. See \citet{LangeLeiss2009} for a detailed analysis of CNF conversion orderings and their worst-case blowup.
\end{itemize}

\subsection{Weighted Finite-State Automata}
\label{sec:wfsa-appendix}
A WFSA $\wfsa$ is a tuple $\wfsaTuple$, where $\alphabet$ is an alphabet, $\states$ is a finite set of states, $\arcs$ is a set of edges (each having form $\edge{\state}{\sym}{\rw}{\statep}$, with $\statep,\state \in \states$, $\rw \in \Weights$ and $\sym \in \alphabet$), $\initf\colon \states \to \Weights$ and $\finalf\colon \states \to \Weights$ are respectively the initial and final weight functions. A \defn{path} $\apath$ is a sequence of consecutive edges
\begin{align}
\edge{\state_0}{\sym_1}{\rw_1}{\state_1} \edge{}{\sym_2}{\rw_2}{\state_2} \cdots \edge{\state_{N-1}}{\sym_N}{\rw_N}{\state_N}
\end{align}
the path's weight is defined as $\weight{\apath}\!=\!\initf{(\state_0)}\cdot \rw_1\cdot \rw_2 \cdots \rw_N \cdot \finalf{(\state_N)}$, the path's yield is $\yield{\apath}\!=\!\sym_1 \cdots \sym_N$. We denote with $\Paths_{\wfsa}(\str)$ the set of paths with yield $\str$. A WFSA defines the weighted language:
\begin{align}
    \wl{\wfsa}{}{\str} =\!\!\!\! \sum_{\apath \in \Paths_{\wfsa}(\str)}\!\!\!\! \weight{\apath}
\end{align}

\clearpage
\section{Deriving the Prefix Grammar by CFG--FST Composition}
\label{sec:composition}

This section explains how we initially \emph{derived} the prefix grammar\footnote{Beyond CFGs, our observation that prefix parsing is just composition with a finite-state transducer (\cref{sec:composition}) suggests that any formalism that efficiently supports such composition is amenable to our general approach. For example, we can, in principle, use composition to derive prefix parsers for tree-adjoining grammars \citep[see, e.g.,][]{nederhof-1999-computational}. Historically, \citet{schabes-joshi-1988-earley} originally presented an $\bigo{N^9}$-time algorithm to perform prefix parsing on tree-adjoining grammars. This runtime was later improved to $\bigo{N^6}$ by \citet{nederhof-1999-computational}. Generalizing our reduction to tree-adjoining grammars would obviate the need for such bespoke algorithms.\looseness=-1} by composition with a weighted finite-state transducer.  Although \cref{def:prefix-grammar-optimized} is a more succinct encoding of the prefix grammar, we believe the derivation as a composition is of independent interest.  We begin by reviewing weighted relations and weighted finite-state transducers, then introduce the \defn{prefix transducer} $\prefixTransducer$.\looseness=-1

\label{sec:weighted-relations}Let $\alphabet$ and $\otheralphabet$ be alphabets. 
A \defn{weighted relation} between the strings $\strings$ and $\kleene{\otheralphabet}$ is a function ${\WRel \colon\!\strings\! \times\!\kleene{\otheralphabet}\!\to\!\Weights}$.  We define the \defn{composition} of $\WRel$ with the language $\Lang\colon
\strings \to \Weights$, as the language $\Lang \circ \WRel\colon \kleene{\otheralphabet} \to \Weights$, defined as follows:
\begin{align}
[\Lang \circ \WRel](\strt) \defeq \sum_{\str \in \strings}\Lang\Parens{\str} \, \WRel (\str, \strt)
\label{def:composition}
\end{align}
\begin{definition}
A \defn{weighted finite-state transducer} (\defn{WFST}) is a tuple $\wfst = \wfstTuple$ where
\begin{itemize}
\item $\alphabet$ is the \defn{input alphabet}
\item $\otheralphabet$ is the \defn{output alphabet}
\item $\states$ is a finite \defn{set of states}
\item $\arcs$ is a \defn{set of transitions} where each \defn{transition} is of the form $(\edge{\state}{\sym:\,\myothersym}{\rw}{\state'})$ with $\state,\state' \in \states$, $\sym \in \alphabet \cup \{\emptystring\}$, $\myothersym \in \otheralphabet \cup \{\emptystring\}$, and $\rw \in \Weights$
\item $\initf\colon \states \to \Weights$ is the \defn{initial weight} function
\item $\finalf\colon \states \to \Weights$ is the \defn{final weight} function. 
\end{itemize}
A \defn{path} $\apath$ is a sequence of transitions of the form:
\begin{align}
\edge{\state_0}{\sym_1{:}\myothersym_1}{\rw_1}{\state_1}
\edge{}{\sym_2{:}\myothersym_2}{\rw_2}
\cdots
\edge{\state_{N-1}}{\sym_N{:}\myothersym_N}{\rw_N}{\state_N},\nonumber
\end{align}
where $N \ge 0$ is its length, 
$\sym_1 \cdots \sym_N$ is its \defn{input yield}, 
$\myothersym_1 \cdots \myothersym_N$ is its \defn{output yield}, and 
$\weight{\apath} \defeq \initf(\state_0) \rw_1 \cdots \rw_N \finalf(\state_N)$ is its \defn{weight}. We denote by $\Pi(\str, \strt)$ the set of paths with input yield $\str$ and output yield $\strt$. Every transducer $\wfst$ defines a weighted relation as follows:
\begin{align}
\wl{\wfst}{}{\str,\strt} \defeq \smashoperator{\sum_{\apath \in \Pi(\str, \strt)}} \weight{\apath}
\end{align}
\end{definition}

The following WFST $\prefixTransducer$ implements the weighted relation $\wl{\prefixTransducer}{}{\str, \strt} = \indicator{\strt \preceq \str}$.

\begin{definition}
\label{def:prefix-transducer}
The \defn{prefix transducer} $\prefixTransducer$ for a given alphabet $\alphabet$ is illustrated below:\footnote{Note that any FST that encodes $\wl{\prefixTransducer}{}{\str, \strt} = \indicator{\strt \preceq \str}$ would work. Of course, a smaller machine will generally lead to a smaller increase in the size of the transformed grammar.  We believe this transducer is among the smallest possible, as it appears that at least two states are necessary.}
\begin{align}
\prefixTransducer \defeq
\begin{cases}
\begin{tikzpicture}[>=stealth,node distance=2.5cm,on grid,auto]
\tikzstyle{every state}=[fill={rgb:black,1;white,30}]
\node[state, initial below] (copy) {{$\copyState$}};
\node[state,initial below, accepting] (erase) [right=of copy] {$\eraseState$};
\path[->]
(copy) edge [loop above] node {${\color{StringColor}\sigma}{:}{\color{StringColor}\sigma}/1$} (copy)
       edge node {${\color{StringColor}\sigma}{:}{\color{StringColor}\sigma}/1$} (erase)
(erase) edge [loop above] node {${\color{StringColor}\sigma}{:}\emptystring/1$} (erase);
\end{tikzpicture}
\end{cases}
\end{align}
More precisely, the weighted transducer is defined by the tuple $\Tuple{\alphabet,\alphabet, \states, \arcs, \initf, \finalf}$, where
\begin{itemize}
\item $\alphabet$ is both the input and output alphabet
\item the set of states is $\states=\{ \copyState, \eraseState\}$; we call $\copyState$ and $\eraseState$ the \defn{copy} and \defn{erase} states, respectively
\item the transitions $\arcs$ are as shown in the picture
\item the initial weights $\initf$ are $\initf(\copyState) = 1$ and  $\initf(\eraseState) = 1$
\item the final weights are $\finalf(\copyState) = 0$ and $\finalf(\eraseState) = 1$
\end{itemize}
Note that in the picture, the \emph{start} arrow and the double circle mark states with nonzero initial and final weights, respectively. Additionally, a transition labeled ${\color{StringColor}\sigma}$ is shorthand for a set of transitions---one for each symbol in the alphabet.\looseness=-1
\end{definition}

\citet{pasti-etal-2023-intersection} provide a general construction for the 
composition of a CFG $\grammar$ and a WFST $\wfst$, denoted
$\grammar \circ \wfst$, which is itself a CFG, and is \emph{correct} in the 
sense that it denotes the composition between the corresponding weighted relation and language, i.e., $\wl{\grammar \circ \wfst}{} = \wl{\grammar}{} \circ \wl{\wfst}{}$.

Equivalently, we can construct a grammar for the prefix language via composition with $\prefixTransducer$: $\wl{\grammar \circ \prefixTransducer}{} = \prefixLanguage{\wl{\grammar}{}}$ by \citet{pasti-etal-2023-intersection}. The composition construction is spelled out below.

\begin{definition}
\label{def:prefix-grammar-composition}
The \defn{composition-based prefix grammar construction} works as follows.  Given a grammar $\grammar=\wcfgtuple$, the composition $\grammar \circ \prefixTransducer = \Tuple{\set{ \tri{s}{\gramsym}{s'} \mid \gramsym \in \nonterms \cup \alphabet, s,s' \in \set{\copyState, \eraseState}}, \alphabet, \preStart, \primeRules}$ where $\primeRules$ is defined by the rules below.
\begin{subequations}
\label{eq:prefix-grammar-composition}
\begin{align}
&\wproduction{\preStart}{\tri{\copyState}{\start}{\eraseState}}{1} \label{eq:prefixgrammar-start-1} && \\ 
&\wproduction{\preStart}{\tri{\eraseState}{\start}{\eraseState}}{1} \label{eq:prefixgrammar-start-2} && \\
&\wproduction{\tri{\copyState}{\ntX}{\copyState}}{\tri{\copyState}{\gramsym_1}{\copyState}\cdots \tri{\copyState}{\gramsym_K}{\copyState}}{\rw} \label{eq:prefixgrammar-qrules} && \annotation{\wproduction{\ntX}{\gramsym_1 \cdots \gramsym_K}{\rw} \in \rules} \\
&\wproduction{\tri{\copyState}{\sym}{\copyState}}{\sym}{1} \label{eq:prefixgrammar-qterminal} && \annotation{\sym \in \alphabet} \\
& \wproduction{\tri{\copyState}{\ntX}{\eraseState}}{
\tri{\copyState}{\gramsym_1}{\copyState}
\cdots
\tri{\copyState}{\gramsym_{k-1}}{\copyState}
\,
\tri{\copyState}{\gramsym_k}{\eraseState}
\,
\tri{\eraseState}{\gramsym_{k+1}}{\eraseState}
\cdots
\tri{\eraseState}{\gramsym_K}{\eraseState}
}{\rw} \label{eq:prefixgrammar-qmixed}\hspace{-1in} && \\
&   &&\annotation{\wproduction{\ntX}{\gramsym_1 \cdots \gramsym_K}{\rw} \in \rules,\, k \in 1, {\ldots}, K}
\nonumber \\
&\wproduction{\tri{\copyState}{\sym}{\eraseState}}{\sym}{1} \label{eq:prefixgrammar-qstarterminal} && \annotation{\sym \in \alphabet} \\
&\wproduction{\tri{\eraseState}{\ntX}{\eraseState}}{\tri{\eraseState}{\gramsym_1}{\eraseState}\cdots \tri{\eraseState}{\gramsym_K}{\eraseState}}{\rw} \label{eq:prefixgrammar-qstar} && \annotation{\wproduction{\ntX}{\gramsym_1 \cdots \gramsym_K}{\rw} \in \rules} \\
&\wproduction{\tri{\eraseState}{\sym}{\eraseState}}{\emptystring}{1} \label{eq:prefixgrammar-nullary} && \annotation{\sym \in \alphabet} 
\end{align}
\end{subequations}
\end{definition}
\noindent The construction is based directly on \citeposs{pasti-etal-2023-intersection} composition construction for $\grammar \circ \prefixTransducer$ where $\prefixTransducer$ is given in \cref{def:prefix-transducer}.\footnote{\label{fn:minor-optimization}Readers familiar with the composition algorithm may recognize that \cref{def:prefix-grammar-composition} is a specialized application of the composition algorithm for the prefix transducer.  We have hard-coded the prefix transducer's states and transitions.  We have omitted nonterminals of the form $\Tuple{\eraseState, \_, \copyState}$, as they are always nongenerating and, thus, can be removed without altering the weighted language of the prefix grammar.
Additionally, we have exploited knowledge of the path structure in $\prefixTransducer$; specifically, each path with nonzero weight has the form:
$\edge{\copyState}{{\color{StringColor}\sigma}_1{:}{\color{StringColor}\sigma}_1}{1}{\copyState}
\cdots
\edge{\copyState}{{\color{StringColor}\sigma}_{m-1}{:}{\color{StringColor}\sigma}_{m-1}}{1}{\copyState}
\edge{}{{\color{StringColor}\sigma}_m{:}{\color{StringColor}\sigma}_m}{1}{\eraseState}
\edge{}{{\color{StringColor}\sigma}_{m+1}{:}\emptystring}{1}{\eraseState}
\cdots
\edge{\eraseState}{{\color{StringColor}\sigma}_M{:}\emptystring}{1}{\eraseState}$.
We exploit this knowledge to reduce the number of useless rules; we can see this structure in the right-hand side of \cref{eq:prefixgrammar-qmixed}.\looseness=-1
} The proposition below establishes that the composition-based construction also correctly encodes the prefix language.\looseness=-1
\begin{restatable}{proposition}{restatablePrefixCompositionCorrectness}
\label{thm:prefix-correctness}
Let $\grammar$ be a CFG. Then, $\grammar \circ \prefixTransducer$ correctly encodes the prefix language of $\grammar$, i.e., $\wl{\grammar \circ \prefixTransducer}{} = \prefixLanguage{\wl{\grammar}{}}$.
\end{restatable}
\begin{proof}
Let $\grammar=\wcfgtuple$.
\begin{subequations}
\begin{align}
\wl{\grammar \circ \prefixTransducer}{}{\str}
&= \sum_{\strz \in \strings} \wl{\grammar}{}{\strz} \cdot \wl{\prefixTransducer}{}{\strz, \str}
&\text{\Comment{By Corollary 1 of \citet{pasti-etal-2023-intersection}}}
\\
&= \sum_{\strz \in \strings} \wl{\grammar}{}{\strz} \cdot \indicator{\str \preceq \strz}
&\text{\Comment{By the definition of the prefix transducer (\cref{def:prefix-transducer})}}
\\
&= \prefixLanguage{\wl{\grammar}{}}(\str)
&\text{\Comment{By the definition of prefix language (\cref{sec:background})}}
\end{align}
\end{subequations}
\end{proof}
\looseness=-1

\clearpage

\section{Instantiations of Prefix Parsing with the Prefix Grammar}
\label{sec:app-prefix-parsers}
In \cref{sec:prefix-parsing}, we presented prefix parsing as a general transformation that allows us to turn any existing parsing algorithm into a prefix parsing one; we shall now consider two specific instantiations of this approach, one with CKY and the other with Earley as the underlying parser.

\subsection{Prefix Parsing with \IncrCKY}
\label{sec:prefix-cky}

We first instantiate the prefix-grammar approach with CKY. Our incremental formulation, \IncrCKY (\cref{alg:incremental-cky}), runs in $\bigo{|\EnsureCNF(\grammar)|N^3}$, where $\EnsureCNF(\grammar)$ is the CNF conversion of $\grammar$. \PrefixCKY (\cref{alg:cky-prefix-parser}) applies the prefix grammar transformation and then calls \IncrCKY. Note that the runtime of \IncrCKY follows the form of \cref{thm:prefix-parsing-runtime}. We next analyze the blowup introduced by the prefix grammar transformation and all the preprocessing steps required in \PrefixCKY.

\begin{proposition}
\label{prop:prefix-cnf-bound}
Let $\grammar$ be a CFG.  Then $|\EnsureCNF(\PrefixOperator(\EnsureCTF(\grammar)))| = \bigo{|\grammar|^2+ |\nonterms||\grammar|}$. 
\end{proposition}
\begin{proof}
We trace the grammar size and the nonterminal count through the preprocessing pipeline, referencing the bounds established in \cref{sec:wcfg-appendix}.  We write $G_i$ for the grammar after step $i$, and $N_i$ for its nonterminal set.
Given the input grammar $\grammar=\wcfgtuple$, we have:
\begin{enumerate}
\item $\EnsureCTF$: $|G_1'| \le 3|\grammar|$,\; $|N_1'| \le |\nonterms| + |\grammar| - |\rules|$ \hfill (\cref{eq:bound-binarize})
\item $\PrefixOperator$: $|G_2'| \le \tfrac{8}{3}|G_1'| + 3 \le 8|\grammar| + 3$,\; $|N_2'| = 2|N_1'| + 1$ \hfill (\cref{prop:optimized-grammar-bound}, \cref{def:prefix-grammar-optimized})

The nonterminal count doubles because the prefix grammar introduces a prime copy $\primeNT{\ntX}$ of each nonterminal $\ntX \in N_1'$, plus one prefix start symbol $\preStart$. Moreover, because $\PrefixOperator$ preserves arity $\le 2$, $G_2'$ is already binary; invoking \cref{eq:bound-cnf} on $G_2'$ would pay for an $\EnsureCTF$ step that is in fact a no-op, incurring a spurious $3\times$ factor via \cref{eq:bound-binarize}. We therefore trace the remaining \EnsureCNF steps directly.

\item $\EnsureNullaryFree$: $|G_3'| \le \tfrac{7}{3}|G_2'| + 3 \le \tfrac{7}{3}(8|\grammar| + 3) + 3 \le 19|\grammar| + 10$,\; $|N_3'| \le |N_2'| + 1$ \hfill (\cref{eq:bound-nullary})
\item $\EnsureUnaryFree$: $|G_4'| \le |N_3'| \cdot |G_3'| \le (2|N_1'| + 2)(19|\grammar| + 10)$,\; $|N_4'| \le |N_3'|$ \hfill (\cref{eq:bound-unary})
\item $\EnsureTerminalSep$: $|G_5'| \le |G_4'| + 2|\alphabet|$,\; $|N_5'| \le |N_4'| + |\alphabet|$ \hfill (\cref{eq:bound-terminal-sep})
\end{enumerate}

\noindent So $|\EnsureCNF(\PrefixOperator(\EnsureCTF(\grammar)))| = |G_5'| \le (2(|\nonterms| + |\grammar| - |\rules|) + 2)(19|\grammar| + 10) + 2|\alphabet|$. Expanding and dropping lower-order terms: $|G_5'| = \bigo{|\grammar|^2 + |\nonterms||\grammar|}$, matching the proposition.

\end{proof}

We note that this is a worst-case bound; in practice, the blowup is much more contained. In particular, in our experiments with PCFGs, when comparing the size of the preprocessed prefix grammar $\EnsureCNF(\prefixGrammar)$ to that of the preprocessed original grammar $\EnsureCNF(\grammar)$, the size ratio is only a small multiplicative factor  (\cref{tab:grammar-size-cky}). Since \IncrCKY's runtime is linear in the preprocessed grammar size, this small factor translates directly to the runtime overhead for \PrefixCKY.

\begin{algorithm}
\begin{algorithmic}[1] 
\footnotesize
\Function{$\PrefixCKY(\grammar,\chars_1 {\cdots} \chars_N)$}{}
\LinesComment{Apply the prefix-grammar transformation (should be cached for efficiency), and parse as usual.}
\State \Return $\IncrCKY(\PrefixOperator(\EnsureCTF(\grammar)), \chars_1 {\cdots} \chars_N)$
\EndFunction
\end{algorithmic}
\caption{ \footnotesize Prefix parsing with the prefix grammar transformation and \IncrCKY (\cref{alg:incremental-cky}).}
\label{alg:cky-prefix-parser}
\end{algorithm}

\begin{algorithm}[h]
\begin{algorithmic}[1]
\footnotesize
\Function{$\IncrCKY(\grammar,\chars_1 {\cdots} \chars_N)$}{}
\State $\grammar \gets \EnsureCNF(\grammar)$ \Comment{CNF conversion; if needed (should be cached for efficiency)}
\State $\inside{}{N} \gets \texttt{defaultdict}(0)$ \Comment{Initialize new column}
\If{$N=0$} \Comment{Base case: empty string.}
\For{$(\wproduction{\start}{\emptystring}{\rw}) \in \rules$}
\State $\inside{}{N}{0, \start} \opluseq \rw$
\EndFor 
\State \Return $\inside{}{N}(0,\start), \Tuple{\inside{}{N}}$ \Comment{Below: recurse on prefix}
\EndIf
\State $\_, \Tuple{\inside{}{0}, {\ldots}, \inside{}{N-1}} \gets \IncrCKY(\grammar,\chars_1 {\cdots} \chars_{N-1})$ 
\For{$(\wproduction{\ntX}{\chars_N}{\rw}) \in \rules$} \Comment{Preterminal rules $\chars_N$}
    \State $\inside{}{N}{N\!-\!1, \ntX } \opluseq \rw$
\EndFor
\For{$ i \textbf{ in } N\!-\!1\ldots 0$}  \Comment{Iterate over the start point}
    \For{$j \textbf{ in } i \!+\! 1 \ldots N\!-\!1$} \Comment{Iterate over the split point}
        \For{$(\wproduction{\ntX}{\ntY\, \ntZ}{\rw}) \in \rules$} \Comment{Binary rules}
            \State $\inside{}{N}{i, \ntX} \opluseq  \rw \cdot \inside{}{j}{i, \ntY} \cdot \inside{}{N}{j, \ntZ}$
        \EndFor
    \EndFor
\EndFor 
\State \Return $\inside{}{N}(0,\start),\Tuple{\inside{}{0}, {\ldots}, \inside{}{N-1}, \inside{}{N}}$
\EndFunction
\end{algorithmic}
\caption{\footnotesize An incremental, left-to-right formulation of \IncrCKY that supports memoization across prefixes. \IncrCKY returns the full column of inside weights: the entry $\inside{}{j}{i, \ntX}$ gives the total weight of all derivations of $\ntX$ with yield $\chars_{i+1} {\cdots} \chars_j$. The goal weight $\wl{\grammar}{}{\chars_1 {\cdots} \chars_N}$ is $\inside{}{N}{0,\start}$.}
\label{alg:incremental-cky}
\end{algorithm}

\begin{table}[h]
\centering
\footnotesize
\caption{\footnotesize Grammar size of different PCFGs used in our experiments. $\grammar_1$ is the grammar after \IncrCKY's preprocessing; $\grammar_2$ is the grammar after \PrefixCKY's preprocessing. The ratio column reports $|\grammar_2|/|\grammar_1|$. The grammars reported in the table are WSJ 500 (obtained from the first 500 sentences of the Wall Street Journal Corpus) and ``\emph{Social Discourse}'', both obtained from \citep{luong2013-parsing-discourse}. We were not able to include the larger WSJ 5000 because our machine had insufficient memory for the unary removal operation of \EnsureCNF.}
\label{tab:grammar-size-cky}
\renewcommand{\arraystretch}{1.4}
\begin{tabular}{|l@{\hspace{1em}}|r@{\hspace{1em}}|r@{\hspace{1em}}|r@{\hspace{1em}}|c|}
\hline
\textbf{Grammar} & $|\grammar|;|\nonterms|$ & $|\grammar_1|;|\nonterms_1|$ & $|\grammar_2|;|\nonterms_2|$ & \textbf{Ratio} $|\grammar_2|/|\grammar_1|$ \\
\hline
WSJ 500 & 12{,}573; 69 & 73{,}241; 1{,}766  & 235{,}459; 1{,}814 & $3.22\times$ \\
Social Discourse [Sparse]      & 72{,}712; 233 & 211{,}015; 233 & 357{,}066; 306 & $1.69\times$ \\
\hline
\end{tabular}
\renewcommand{\arraystretch}{1.0}
\end{table}

\clearpage
\subsection{Prefix Parsing with Earley's Algorithm}
\label{sec:prefix-earley}
We now instantiate the prefix-grammar approach using the weighted Earley's algorithm \citep{Earley}.  Analogously to \PrefixCKY, \PrefixEarley  (\cref{alg:earley-prefix-parser}) applies the prefix grammar transformation and then parses with Earley's.
Importantly, the experimental results (\cref{fig:runtime-comparison-three-algorithms}) show that prefix parsing is slower than ordinary parsing by a multiplicative factor of $\approx 3$; this matches the prediction from the relative grammar-size blow-up (\cref{tab:grammar-size-earley}).

\begin{algorithm}[h]
\begin{algorithmic}[1]
\footnotesize
\Function{$\PrefixEarley(\grammar,\chars_1 {\cdots} \chars_N)$}{}
\LinesComment{Apply the prefix-grammar transformation (should be cached for efficiency), and parse as usual.}
\State \Return $\Earley(\PrefixOperator(\EnsureCTF(\grammar)), \chars_1 {\cdots} \chars_N)$
\EndFunction
\end{algorithmic}
\caption{\footnotesize Prefix parsing with the prefix grammar transformation and \Earley (\cref{alg:earley}).}
\label{alg:earley-prefix-parser}
\end{algorithm}

\paragraph{An optimized version of Earley's algorithm.}
We present an optimized, incremental version of Earley's algorithm (\cref{alg:earley}),
which we use as a base for our prefix parsing and next-token weight vector algorithms.
Assuming $\grammar$ is nullary-free and unary-cycle-free, our algorithm runs in $\bigo{|\grammar|N^3}$---a factor of $|\rules|$ faster than the traditional dotted-item Earley formulation.\footnote{Our items record only the unmatched suffix $\valpha$, rather than the traditional $\ntX \to \boldsymbol{\gamma}\bullet\valpha$ dotted-item form. This collapses many rules into one item, a simpler alternative to \citeposs{opedal-etal-2023-efficient} approach, which introduces additional item types.}
To achieve this improved runtime, we maintain a list of \emph{items} of the form $\langle i, \ntX/\valpha \rangle$, where $i$ is the starting position and $\ntX/\valpha$ is read as ``$\ntX$ missing $\valpha$ to complete''. The algorithm proceeds through three operations: 
\textsc{Predict} generates new items from grammar rules.\footnote{Our implementation uses the left-corner relation to prune the \textsc{Predict} step, skipping rules whose left-hand sides are not needed by any active item \citep{stolcke-1995-efficient}.}
\textsc{Scan} advances items over terminal symbols, and \textsc{Attach}---traditionally also known as \emph{complete}---attaches completed items to items awaiting them. The data structure $\EarleyW{k}$ indexes items at position $k$ by the symbol they are waiting for, enabling efficient lookup during the \textsc{Attach} phase.

\paragraph{Grammar blowup: bounds and empirical evidence.} Following \cref{thm:prefix-parsing-runtime}, we bound $|\phi(\prefixGrammar)|$, where $\phi$ is the preprocessing pipeline required for \PrefixEarley.

\begin{proposition}
Let $\grammar$ be a CFG. Then,\\ $|\EnsureUnaryCycleFree(\EnsureNullaryFree(\PrefixOperator(\EnsureCTF(\grammar))))| = \bigo{|\grammar|^2+|\nonterms|^2}$.\footnote{The $|\grammar|^2$ term arises when $\grammar$ has high-arity rules: $\EnsureCTF$ introduces up to $|\grammar| - |\rules|$ new nonterminals, whose primes may participate in unary cycles. For grammars already in canonical two-form, the bound tightens to $\bigo{|\grammar| + |\nonterms|^2}$. This worst case is loose in practice (see \cref{tab:grammar-size-earley}).}
\end{proposition}
\begin{proof}
We analyze $\phi(\grammar)$ first, then extend to $\phi(\prefixGrammar)$.
We trace the grammar size and the nonterminal count through the preprocessing pipeline, referencing the bounds established in \cref{sec:wcfg-appendix}.  We write $G_i$ for the grammar after step $i$, and $N_i$ for its nonterminal set.
We apply $\EnsureCTF$ (required before $\EnsureNullaryFree$ for efficiency), $\EnsureNullaryFree$, and $\EnsureUnaryCycleFree$ in sequence. Given the input grammar $\grammar=\wcfgtuple$, we have
\begin{enumerate}
\item $\EnsureCTF$: $|G_1| \le 3|\grammar|$,\; $|N_1| \le |\nonterms| + |\grammar| - |\rules|$ \hfill (\cref{eq:bound-binarize})
\item $\EnsureNullaryFree$: $|G_2| \le  7|\grammar| + 3$,\; $|N_2| \le |\nonterms| + |\grammar| - |\rules| + 1$ \hfill (\cref{eq:bound-nullary})
\item $\EnsureUnaryCycleFree$: $|G_3| \le 7|\grammar| + 3 + 2 (|\nonterms| + |\grammar| - |\rules|)^2 $  
\end{enumerate}
Thus, $|\phi(\grammar)| \le 7|\grammar| + 3 + 2 (|\nonterms| + |\grammar| - |\rules|)^2$. For the prefix grammar, the preprocessing pipeline becomes:
\begin{enumerate}
\item $\EnsureCTF$: $|G_1| \le 3|\grammar|$,\; $|N_1| \le |\nonterms| + |\grammar| - |\rules|$ \hfill (\cref{eq:bound-binarize})
\item $\PrefixOperator$: $|G_2| \le \tfrac{8}{3}|G_1| + 3 \le 8|\grammar| + 3$,\; $|N_2| = 2|N_1| + 1$ \hfill (\cref{prop:optimized-grammar-bound}, \cref{def:prefix-grammar-optimized})

Since $\PrefixOperator$ preserves arity $\le 2$, $G_2$ is already binary and binarization is a no-op.
\item $\EnsureNullaryFree$: $|G_3| \le \tfrac{7}{3}|G_2| + 3 \le \tfrac{7}{3}(8|\grammar| + 3) + 3 \le 19|\grammar| + 10$,\; $|N_3| \le |N_2| + 1 = 2(|\nonterms| + |\grammar| - |\rules|) + 2$ \hfill (\cref{eq:bound-nullary})
\item $\EnsureUnaryCycleFree$: $|G_4| \le |G_3| + 2|N_3|^2 \le 19|\grammar| + 10 + 2(2(|\nonterms| + |\grammar| - |\rules|) + 2)^2$ \hfill (\cref{eq:bound-unary-cycle})
\end{enumerate}
Thus, $|\phi(\prefixGrammar)| \le 19|\grammar| + 10 + 2(2(|\nonterms|+|\grammar|-|\rules|)+2)^2$. Expanding and dropping lower-order terms: $|\phi(\prefixGrammar)| = \bigo{|\grammar|^2 + |\nonterms|^2}$, matching the proposition.
\end{proof}

In practice, this bound is far from tight. On the grammars we experimented with, the preprocessed prefix grammar is only a small multiplicative factor larger than the ordinary grammar preprocessed for \Earley (\cref{tab:grammar-size-earley}). Given our optimized \Earley implementation's runtime bound, this small factor translates directly to the ratio of prefix parsing to parsing runtime, as confirmed empirically by our experiments (\cref{fig:runtime-comparison-three-algorithms}).\looseness=-1

\begin{algorithm}[h]
\caption{\footnotesize Earley's Algorithm takes the grammar $\grammar = \wcfgtuple$ and a string $\str = \chars_1 {\cdots} \chars_k \in \strings$. It returns the \emph{inside weight columns} $\Tuple{\inside{}{0}, {\ldots}, \inside{}{k}}$ together with the waiting-for dictionaries $\Tuple{\EarleyW{0}, {\ldots}, \EarleyW{k}}$; the total string weight is $\wl{\grammar}{}{\chars_1 {\cdots} \chars_k} = \inside{}{k}{0, \start}$. The notation $\ntX/\valpha$ is read as $\ntX$ missing $\valpha$ to complete.
The data structure $\EarleyW{k}[\gramsym]$ maintains a dictionary of items $\Tuple{i, \ntX / \gramsym_1 \valpha_{[1:]}}$ that are awaiting the symbol $\gramsym$ in order to move towards completion. The queue $\EarleyQ$ is a priority queue that prioritizes items with the smallest span $k-i$, where $k$ is the current column, and $i$ is the item's starting position.
}
\label{alg:earley}
\begin{algorithmic}[1]
\footnotesize
\Func{$\Earley(\grammar,\chars_1 {\cdots} \chars_k)$}
\Comment{The \Earley function should be memoized for efficiency.}
\State $\grammar \gets \EnsureUnaryCycleFree(\EnsureNullaryFree(\grammar))$ \Comment{Ensure no nullary rules, or unary rule cycles; memoize}
\State $\inside{}{k} \gets \texttt{defaultdict}(0)$ \Comment{Initialize inside weight column}
\State $\EarleyW{k} \gets \texttt{defaultdict}(\texttt{set})$
\Comment{Initialize the waiting-for dictionary}
\State $\EarleyQ \gets \texttt{priority\_queue}()$ \Comment{This queue prioritizes items with the shortest span.}

\If{$k = 0$} \Comment{Base cases}
\For{$(\wproduction{\ntX}{\valpha}{\rw}) \in \rules$}
\State $\chartupdate{0}{0, \ntX/\valpha}{\rw}$
\EndFor
\State \Return $\Tuple{\inside{}{0}}, \Tuple{\EarleyW{0}}$
\EndIf

\State $\Tuple{\inside{}{0}, {\ldots}, \inside{}{k-1}}, \Tuple{\EarleyW{0}, {\ldots}, \EarleyW{k-1}} \gets \Earley(\grammar,\chars_1 {\cdots} \chars_{k-1})$ 
\Comment{Recurse on prefix}

\For{$\Tuple{i, \ntX/\chars_k \valpha} \in \EarleyW{k-1}[\chars_k]$}   \Comment{\textsc{Scan}} \label{alg:line-earley-scan}
\State $\chartupdate{k}{i, \ntX / \valpha}{\inside{}{k-1}{i, \ntX / \chars_k \valpha}}$
\EndFor
\While{$\EarleyQ$}\Comment{\textsc{Attach}} \label{alg:line-earley-attach}
\State $\Tuple{j, \ntY/\emptystring} \gets \EarleyQ.\texttt{pop()}$ \Comment{Pop completed item $\ntY$ together with the index $j$ where item began}
\For{$\Tuple{i, \ntX/\ntY \valpha} \in \EarleyW{j}[\ntY]$} \Comment{Iterate through items ending at $j$ that are waiting for a $\ntY$}
\State $\chartupdate{k}{i, \ntX / \valpha}{\inside{}{j}{i, \ntX / \ntY \valpha}  \cdot \inside{}{k}{j, \ntY / \emptystring}}$ \Comment{Attach the weights and update}
\EndFor
\EndWhile
\For{$(\wproduction{\ntX}{\valpha}{\rw}) \in \rules$} \Comment{\textsc{Predict}}
\label{alg:line-earley-predict}
\State $\chartupdate{k}{k, \ntX/\valpha}{\rw}$
\EndFor
\State \Return $\Tuple{\inside{}{0}, {\ldots}, \inside{}{k-1}, \inside{}{k}}, \Tuple{\EarleyW{0}, {\ldots}, \EarleyW{k-1}, \EarleyW{k}}$
\EndFunc
\Statex \vspace{-.5em}
\LinesComment{Helper methods}
\Func{\fbox{\mbox{\ensuremath{\chartupdate{k}{i, \ntX / \valpha}{\val}}}}} \Comment{This method updates the weights, $\EarleyQ$ and $\EarleyW{k}$}
\label{alg:line-update}
    \If{$\valpha = \emptystring$} $\EarleyQ.\texttt{push}(\Tuple{i,\ntX/\emptystring})$  \Comment{Newly completed items get scheduled in the priority queue} \label{alg:line-update-completed}
    \Else\ $\EarleyW{k}[\gramsym_{1}].\texttt{add}(\Tuple{i, \ntX/\valpha})$  \Comment{Item $\langle i, \ntX/\valpha \rangle$ is waiting for $\gramsym_1$ to move forward} \label{alg:line-update-waiting-for}
    \EndIf
    \State $\inside{}{k}{i, \ntX / \valpha} \gets \inside{}{k}{i, \ntX / \valpha} + \val$
\EndFunc
\end{algorithmic}
\end{algorithm}

\begin{table}[h]    
\centering                  
\footnotesize                                                                
\caption{\footnotesize Grammar size of different PCFGs $\grammar$ used in our
experiments. $\grammar_1$ and $\grammar_2$ are respectively the grammars     
obtained after the preprocessing step of $\Earley$ and $\PrefixEarley$. The ratio column reports $|\grammar_2|/|\grammar_1|$. The grammars reported in the table are WSJ 500 (obtained from the first 500 sentences of the Wall Street Journal Corpus), WSJ 5000 (obtained from the first 5000 sentences of the Wall Street Journal Corpus), and ``\emph{Social Discourse}''. We empirically found that applying $\EnsureCNF$ after the prefix grammar transformation and right before $\EnsureNullaryFree$ yielded a smaller grammar; the results shown in this table follow this approach.}       
\label{tab:grammar-size-earley}                                              
\renewcommand{\arraystretch}{1.4}
\begin{tabular}{|l@{\hspace{1em}}|r@{\hspace{1em}}|r@{\hspace{1em}}|r@{\hspace{1em}}|c|}                                                                  
\hline
\textbf{Grammar} & $|\grammar|;|\nonterms|$ & $|\grammar_1|;|\nonterms_1|$ & 
$|\grammar_2|;|\nonterms_2|$ & \textbf{Ratio} $|\grammar_2|/|\grammar_1|$ \\ 
\hline
WSJ 500  & 12{,}573; 69   & 15{,}981; 1{,}775    & 43{,}701; 5{,}201    &    
$2.73\times$ \\                                                                                      WSJ 5000 & 116{,}667; 448 & 177{,}303; 30{,}878  & 494{,}017; 84{,}566  &
$2.79\times$ \\
Social Discourse [Sparse]   & 72{,}712; 233  & 72{,}712; 233        & 143{,}548; 435       &
$1.97\times$ \\ 
\hline                                                          
\end{tabular}                                                                
\renewcommand{\arraystretch}{1.0}                               
\end{table}

\clearpage
\section{Practical Instantiations of the Next-Token Weight Vector Algorithm}
\label{sec:app-next-token}

In \cref{sec:next-token-algorithms}, we presented a general framework for computing the next-token weight vector $\nextTokenWeights{}{\str}$ via lattice parsing and algorithmic differentiation. This section instantiates this framework with two concrete parsing algorithms: CKY (\cref{sec:app-next-token-cky}) and Earley's (\cref{sec:app-next-token-earley}). In both cases, we derive the gradient algorithm by manually applying the rules of reverse-mode algorithmic differentiation to the corresponding lattice parser.\footnote{In general, the gradient program can also be obtained automatically, e.g., via program tracing, operator overloading, or source-to-source transformation \citep[see][for an overview]{griewank}, using tools such as PyTorch \citep{pytorch}, TensorFlow \citep{tensorflow}, JAX \citep{jax}, Tapenade \citep{tapenade}, or Dyna \citep{eisner-goldlust-smith-2005}.}\looseness=-1

\subsection{Next-Token Weight Vector with CKY}
\label{sec:app-next-token-cky}

\begin{algorithm*}[h]
\footnotesize
\begin{minipage}[t]{0.5\textwidth}
\begin{algorithmic}[1]
\Function{$\IncrCKYLattice(\grammar,\chars_1 {\cdots} \chars_N, \parameters)$}{}
\State $\grammar \gets \EnsureCNF(\grammar)$ \Comment{CNF conversion; memoize}
\State $\Tuple{\inside{}{0}, {\ldots}, \inside{}{N}} \gets \IncrCKY(\grammar,\chars_1 {\cdots} \chars_{N})$
\State $\z \gets \texttt{defaultdict}(0)$ \Comment{Initialize new column}
\LinesComment{Base case:}
\For{$(\wproduction{\ntX}{\sym}{\rw}) \in \rules$ with $\sym \in \alphabet$} \label{alg:line-terminal}
    \State $\z(N, \ntX) \opluseq \rw \cdot \param_{\sym}$
\EndFor
\LinesComment{Recursive step:}
\For{$i \textbf{ in } N, {\ldots}, 0$} \label{alg:line-forward-inside} \Comment{Iterate over the start point.}
    \For{$j \textbf{ in } i \!+\! 1 \ldots N$} \Comment{Split point}
        \For{$(\wproduction{\ntX}{\ntY\, \ntZ}{\rw}) \in \rules$} \Comment{Binary rules}
            \State $\z(i, \ntX) \opluseq  \rw \cdot \inside{}{j}{i, \ntY} \cdot \z(j, \ntZ)$
        \EndFor
    \EndFor
\EndFor
\State \Return $\z(0, \start)$ \Comment{equals $\partition{\chars}(\parameters)$}
\EndFunction
\end{algorithmic}
\end{minipage}\hfill
\begin{minipage}[t]{0.5\textwidth}
\begin{algorithmic}[1]
\Function{$\NextTokenCKY(\grammar,\chars_1 {\cdots} \chars_N)$}{}
\LinesComment{Create prefix grammar in CNF; memoize}
\State $\prefixGrammar \gets \EnsureCNF(\PrefixOperator(\EnsureCTF(\grammar)))$ 
\State $\Tuple{\inside{}{0}, {\ldots}, \inside{}{N}} \gets \IncrCKY(\prefixGrammar,\chars_1 {\cdots} \chars_N)$
\State $\dz \gets \texttt{defaultdict}(0)$
\State $\dz(0,\start) \opluseq  1$ \Comment{Base case}
\label{alg:line-base-case}
\For{$i \textbf{ in } 0 \ldots N\!-\!1$} \Comment{Inverse iteration}\label{alg:line-backward}
\For{$j$ \textbf{ in } $i \!+\! 1 \ldots N$} \Comment{Loop over the split point $j$}
\For{$(\wproduction{\ntX}{\ntY\, \ntZ}{\rw}) \in \rules$}
\State $\dz(j,\ntZ) \opluseq \rw \cdot \inside{}{j}{i,\ntY } \cdot \dz(i,\ntX)$
\EndFor
\EndFor
\EndFor
\State $\nextTokenWeights{}{} \gets 0^{\alphabet}$
\For{$(\wproduction{\ntX}{\sym}{\rw}) \in \rules$ where $\sym \in \alphabet$} \label{alg:line-terminal-backward} \Comment{Preterminal rules}
\State $\nextTokenWeights{\sym}{} \opluseq \rw \cdot \dz(N,\ntX)$
\EndFor
\State \Return $\nextTokenWeights{}{}$
\EndFunction
\end{algorithmic}
\end{minipage}
\caption{\footnotesize \textbf{Left}: the \IncrCKYLattice is a specialized lattice parser that parses the next-token lattice $\nextTokenLattice{\str}$ for some choice of $\parameters \in \Weights^{\alphabet}$. Note that \IncrCKYLattice is directly derived from \IncrCKY, differing only in the terminal step (\cref{alg:line-terminal}), which scans all terminal symbols simultaneously, weighting each by $\param_{\sym}$. \textbf{Right}: the \NextTokenCKY algorithm, which was derived by manually applying algorithmic differentiation to \IncrCKYLattice. Given the input grammar $\grammar$, the input string $\chars_1 {\cdots} \chars_N$, and the inside weight columns $\Tuple{\inside{}{0}, {\ldots}, \inside{}{N}}$, it returns the next-token weight vector $\nextTokenWeights{}{\str}$. Note that both algorithms call \IncrCKY as a subroutine to compute the inside weight columns (which, however, can be amortized across subsequent calls when \NextTokenCKY is called incrementally).\looseness=-1}
\label{alg:next-token-cky}
\end{algorithm*}

\noindent In \IncrCKYLattice (\cref{alg:next-token-cky}, left), the \emph{forward values} $\z(i,\ntX)$ for $i = 0, {\ldots}, N$ and $\ntX \in \nonterms$ aggregate partial contributions to $\partition{\str}(\parameters) = \z(0,\start)$. In \NextTokenCKY (right), the corresponding \emph{backward values} $\dz(i,\ntX)\defeq \frac{\partial \z(0,\start)}{\partial \z(i,\ntX)}$ disaggregate this sum and isolate each token's contribution.
Readers familiar with the inside--outside algorithm \citep{insideOutside} will notice that the backward pass is closely related to the \emph{outside algorithm}, following \citeposs{eisner-2016-inside} presentation of the outside algorithm as the \emph{adjoint} of the inside algorithm.
The runtime of \NextTokenCKY, following the analysis of \PrefixCKY, is $\bigo{|\EnsureCNF(\PrefixOperator(\EnsureCTF(\grammar)))| N^3}$, including the cost of computing the \emph{inside weight columns} $\inside{}{0}, {\ldots}, \inside{}{N}$, which can be amortized across incremental computations.\looseness=-1

\subsection{Next-Token Weight Vector with Earley's Algorithm}
\label{sec:app-next-token-earley}

We now adapt the lattice parsing and algorithmic differentiation approach to Earley's algorithm \citep{Earley} and its weighted version \citep{stolcke-1995-efficient}. Earley's algorithm requires less preprocessing than CKY---only \EnsureNullaryFree and \EnsureUnaryCycleFree---and runs in subcubic time for most grammars: linear for deterministic grammars, and quadratic for unambiguous ones.

\paragraph{Next-token weight vector with Earley (\cref{alg:next-token-earley}).}
In \IncrEarleyLattice, the \emph{forward values} $\z(j, \ntY/\emptystring)$ for $j \in \{0, {\ldots}, N\}$ and $\ntY \in \nonterms$ aggregate partial contributions to the goal item $\z(0, \start) = \partition{\chars_1 {\cdots} \chars_N}(\parameters)$. In \NextTokenEarley, the corresponding \emph{backward values} $\dz(j, \ntY/\emptystring) \defeq \frac{\partial \z(0, \start)}{\partial \z(j, \ntY/\emptystring)}$ disaggregate this sum to isolate each token's contribution. The runtime of \NextTokenEarley, following the analysis of \PrefixEarley, is $\bigo{|\EnsureUnaryCycleFree(\EnsureNullaryFree(\PrefixOperator(\EnsureCTF(\grammar))))| N^3}$, including the cost of computing the \emph{inside weight columns} $\inside{}{0}, {\ldots}, \inside{}{N}$, which can be amortized across incremental computations.

\begin{algorithm*}[h]
\begin{minipage}[t]{0.5\textwidth}
\begin{algorithmic}[1]
\footnotesize
\Function{$\IncrEarleyLattice(\grammar,\chars_1 {\cdots} \chars_N, \parameters)$}{}
\State $\grammar \gets \EnsureNullaryFree(\grammar)$ 
\State $\grammar \gets \EnsureUnaryCycleFree(\grammar)$
\State $\wcfgtuple \gets \grammar$
\State $\Tuple{\inside{}{0}, {\ldots}, \inside{}{N}}\!,\!\Tuple{\EarleyW{0}, {\ldots}, \EarleyW{N}} \gets \Earley(\grammar,\chars_1 {\cdots} \chars_{N})$
\State $\z \gets \texttt{defaultdict}(0)$ \Comment{Initialize new column}
\State $\EarleyQ \gets \texttt{priority\_queue}()$

\For{$\sym \in \EarleyW{N}.\texttt{keys} \cap \alphabet$} \Comment{\textsc{Scan}}
  \For{$\Tuple{i, \ntX/\sym} \in \EarleyW{N}[\sym]$} \label{alg:line-scan-optimized}
      \State $\z(i, \ntX / \emptystring) \opluseq \inside{}{N}{i, \ntX/\sym} \cdot \param_{\sym}$
      \State $\EarleyQ.\texttt{push}(\Tuple{i,\ntX/\emptystring})$
  \EndFor
\EndFor
\While{$\EarleyQ$}   \Comment{\textsc{Attach}}
\State $\Tuple{j, \ntY/\emptystring} \gets \EarleyQ.\texttt{pop()}$
\For{$\Tuple{i, \ntX/\ntY } \in \EarleyW{j}[\ntY]$} \label{alg:line-attached-optimized}
  \State $\z(i, \ntX / \emptystring) \opluseq \inside{}{j}{i, \ntX / \ntY }  \cdot \z(j, \ntY / \emptystring)$
  \State $\EarleyQ.\texttt{push}(\Tuple{i,\ntX/\emptystring})$
\EndFor
\EndWhile
\State \Return $\z(0, \start/\emptystring)$ \Comment{equals $\partition{\chars_1 {\cdots} \chars_N}(\parameters)$}
\EndFunction
\end{algorithmic}
\end{minipage}\hfill
\begin{minipage}[t]{0.5\textwidth}
\begin{algorithmic}[1]
\footnotesize
\Func{$\NextTokenEarley(\grammar,\chars_1 {\cdots} \chars_N)$}
\State $\prefixGrammar \gets \PrefixOperator(\EnsureCTF(\grammar))$ \Comment{Prefix transformation}
\State $\Tuple{\inside{}{0}, {\ldots}, \inside{}{N}}\!,\!\Tuple{\EarleyW{0}, {\ldots}, \EarleyW{N}}\!\gets\!\Earley(\prefixGrammar,\chars_1 {\cdots} \chars_{N})$
    \State $\nextTokenWeights{}{} \gets 0^{\alphabet}$
    \For{$\sym \in \EarleyW{N}.\texttt{keys} \cap \alphabet$}
            \For{$\Tuple{i, \ntX /\sym} \in \EarleyW{N}[\sym]$} 
                    \State $\nextTokenWeights{\sym}{} \opluseq \inside{}{N}{i, \ntX / \sym} \cdot
                    \dz(i,\ntX / \emptystring)$ 
            \EndFor
    \EndFor
    \State \Return $\nextTokenWeights{}{}$
\EndFunc
\Func{\ensuremath{\dz(j,\ntY / \emptystring)}} \Comment{Memoize for efficiency} \label{alg:line-backward-helper}
    \If{$(j,\ntY) = (0, \start)$} 
        \Return $1$ \Comment{Base Case}
    \EndIf
    \State $\val \gets 0$
    \For{$\Tuple{i, \ntX / \ntY} \in \EarleyW{j}[\ntY]$} \Comment{Recursive bottom-up call.}
            \State $\val \opluseq \inside{}{j}{i, \ntX / \ntY} \cdot
                      \dz(i,\ntX / \emptystring)$ \Comment{Note: $i \leq j$.}
    \EndFor
    \State \Return $\val$
\EndFunc    
\end{algorithmic}
\end{minipage}
\caption{\footnotesize\textbf{Left}: \IncrEarleyLattice is a lattice parser derived from Earley's algorithm (\cref{alg:earley}) that evaluates the next-token lattice $\nextTokenLattice{\chars}$ for some choice of $\parameters \in \Weights^{\alphabet}$. It differs from \Earley only in the \textsc{Scan} step, which scans all terminal symbols simultaneously, weighting each by $\param_{\sym}$. Since only a single terminal is scanned at position $N\!+\!1$, the following simplifications apply: 
(1)~\textsc{Scan} processes only items from $\EarleyW{N}$ that complete after scanning one symbol (\cref{alg:line-scan-optimized}), and \textsc{Attach} only advances items with a single remaining symbol (\cref{alg:line-attached-optimized}); (2)~since no further columns are needed, \textsc{Predict} can be omitted entirely.
\textbf{Right}: \NextTokenEarley is the gradient algorithm of \IncrEarleyLattice, obtained by applying algorithmic differentiation. Given the prefix grammar $\prefixGrammar$, the input string $\str = \chars_1 {\cdots} \chars_N$, and the inside weight columns from \Earley, it returns the next-token weight vector $\nextTokenWeights{}{\str}$. The backward values $\dz(j, \ntY/\emptystring)$ are computed by the memoized helper function (\cref{alg:line-backward-helper}).
}
\label{alg:next-token-earley}
\end{algorithm*}

\clearpage

\section{Deferred Proofs}
\label{sec:proofs-of-correctness}

\subsection{Proof of \texorpdfstring{\cref{thm:optimized-prefix-grammar-is-correct}}{the prefix-grammar correctness theorem}}

\restatablePrefixGrammarCorrectness*
\begin{proof}
\label{thm:optimized-prefix-grammar-is-correct-proof}
We strengthen the proposition to a claim over all $\gramsym \in \nonterms \cup \alphabet$, prove the strengthening by induction on derivation-tree height in $\prefixGrammar$, and recover the proposition by specializing to $\gramsym = \start$.
\paragraph{Strengthening.}
The \cref{thm:optimized-prefix-grammar-is-correct} is implied by the following statement, for all $\gramsym \in \nonterms \cup \alphabet$ and all $\str \in \kleeneplus{\alphabet}$:
\begin{align}
\wl{\prefixGrammar}{\primeNT{\gramsym}}{\str} 
= \sum_{\strt \in \strings} \wl{\grammar}{\gramsym}{\str \strt} 
\quad\text{and}\quad
\wl{\prefixGrammar}{\gramsym}{\str} 
= \wl{\grammar}{\gramsym}{\str}
\label{eq:stronger-statement}
\end{align}
The latter holds because $\prefixGrammar$'s additional rules have only prime nonterminals or $\preStart$ on the LHS (\cref{def:prefix-grammar-optimized}), so non-prime $\gramsym$-rooted derivations use only the rules of $\rules$. Thus, we need only prove the former, which we re-express self-recursively in $\prefixGrammar$:
\begin{align}
\wl{\prefixGrammar}{\primeNT{\gramsym}}{\str} 
= \sum_{\strt \in \strings} \wl{\prefixGrammar}{\gramsym}{\str \strt} 
\end{align}

\paragraph{$\hookrightarrow$ Hooking up \cref{eq:stronger-statement} to \cref{thm:optimized-prefix-grammar-is-correct}.}
Specializing the first equation of~\cref{eq:stronger-statement} to $\gramsym = \start$ and expanding $\preStart$'s rules (\cref{def:prefix-grammar-optimized}) yields the proposition. Recall that $\prefixGrammar$'s start symbol $\preStart$ has two rules: $\wproduction{\preStart}{\primeNT{\start}}{1}$ and $\wproduction{\preStart}{\emptystring}{\totalWeight{\start}}$.\looseness=-1
\begin{itemize}
\item \textbf{Case ($\str \in \kleeneplus{\alphabet}$)}: Only the first rule contributes (the second requires $\str = \emptystring$). By \cref{eq:stronger-statement} at $\gramsym = \start$:
\[
\wl{\prefixGrammar}{}{\str} = \wl{\prefixGrammar}{\primeNT{\start}}{\str} = \sum_{\strt \in \strings} \wl{\grammar}{\start}{\str \strt} = \prefixLanguage{\wl{\grammar}{}}(\str)
\]

\item \textbf{Case ($\str = \emptystring$)}: Every $\primeNT{\start}$-derivation has nonempty yield (the prime chain from the root must terminate at a terminal via \cref{eq:prefix-grammar-optimized-borderline-exit}), so only the $\preStart \to \emptystring$ rule contributes. By the definition of $\totalWeight{\start}$:
\[
\wl{\prefixGrammar}{}{\emptystring} = \totalWeight{\start} = \sum_{\strt \in \strings} \wl{\grammar}{\start}{\strt} = \prefixLanguage{\wl{\grammar}{}}(\emptystring)
\]
\end{itemize}

\paragraph{Notation.}
We introduce the following notation to keep the proof tidy.
\begin{itemize}
\item Let $\preD$ denote the set of derivation trees of $\prefixGrammar$.
\item Let $\preD^{(h)}$ denote the subset of trees in $\preD$ of height $\le h$; $\preD^{(h)}_{\gramsym}(\str)$ restricts to $\gramsym$-rooted trees with yield $\str$.

\item Let $\grammar^{(h)}$ denote $\grammar$ with its derivations restricted to height $\le h$.
\item Let $\prefixGrammar^{(h)} \defeq \prefixLanguage{(\grammar^{(h)})}$ be the prefix grammar of $\grammar^{(h)}$; its prime rule weights use the height-bounded totals $\preWeightTotal^{(h)}_{\gramsym}$ rather than $\totalWeight{\gramsym}$.
\item Let $\preWeight^{(h)}_{\gramsym}(\str) \defeq \sum_{\tree \in \preD^{(h)}_{\gramsym}(\str) } \weight{\tree}$
\item Let $\preWeightTotal^{(h)}_{\gramsym} \defeq \sum_{\strt \in \strings} \preWeight^{(h)}_{\gramsym}(\strt)$ 
\end{itemize}

\paragraph{Proof strategy.}
We prove \cref{eq:stronger-statement} by induction on the height $h \ge 1$ of derivation trees in $\prefixGrammar$. We define the height-indexed proposition:
\begin{align}
\Phi(h) \Longleftrightarrow \forall \gramsym \in \nonterms \cup \alphabet, \str \in \kleeneplus{\alphabet} \colon
\preWeight^{(h)}_{\primeNT{\gramsym}}(\str)
= \sum_{\strt \in \strings} \preWeight^{(h)}_{\gramsym}(\str \strt)
\tag{IH}\label{IH}
\end{align}

\paragraph{Base case ($h = 1$).}
We show $\Phi(1)$ holds. There are two ways for a tree to have height one.
\begin{itemize}
\item \textbf{Case ($\gramsym \in \alphabet$)}: Since $\primeNT{\gramsym} = \gramsym$, both sides of $\Phi(1)$ reduce to $\indicator{\str = \gramsym}$, using $\str \in \kleeneplus{\alphabet}$ to collapse the right-hand side sum.
\item \textbf{Case ($\wproduction{\ntX}{\emptystring}{\rw}$)}: $\primeNT{\ntX}$ has no height-$1$ derivation since the rules of \cref{eq:prefix-grammar-optimized-borderline-exit} have nonempty right-hand sides, so the LHS of $\Phi(1)$ is $0$. The RHS is also $0$ because $\preWeight^{(1)}_\ntX(\str\strt)$ is nonzero only when $\str\strt = \emptystring$, which contradicts $\str \in \kleeneplus{\alphabet}$.
\end{itemize}

\paragraph{Inductive hypothesis.}
Suppose $\Phi(h)$ holds for some $h \ge 1$.

\paragraph{Inductive case.}
We show $\Phi(h{+}1)$ holds. In this case, we know that $\primeNT{\gramsym}$ is of the form $\primeNT{\ntX}$ for $\ntX \in \nonterms$.

{\footnotesize
\begin{subequations}
\begin{align}
&
\sum_{\strt \in \strings}
\preWeight^{(h+1)}_{\ntX}(\str \strt) \nonumber\\
&= 
\sum_{\strt \in \strings}
\sum_{\wproduction{\ntX}{\gramsym_1 \cdots \gramsym_K}{\rw}}
\sum_{\str \strt = \strv{1} \cdots \strv{K}}
\rw
\cdot
\preWeight^{(h)}_{\gramsym_1}(\strv{1})
\cdots
\preWeight^{(h)}_{\gramsym_k}(\strv{k})
\preWeight^{(h)}_{\gramsym_{k+1}}(\strv{k+1})
\cdots
\preWeight^{(h)}_{\gramsym_K}(\strv{K}) \\
&= 
\sum_{
\wproduction{\ntX}{\gramsym_1 \cdots \gramsym_K}{\rw}
}
\sum_{k=1}^K
\sum_{\substack{
\str = \strv{1} \cdots \strv{k} \\
\strt = \strv{k+1} \cdots \strv{K}
}}
\rw
\cdot
\preWeight^{(h)}_{\gramsym_1}(\strv{1})
\cdots
\preWeight^{(h)}_{\gramsym_k}(\strv{k})
\preWeight^{(h)}_{\gramsym_{k+1}}(\strv{k+1})
\cdots
\preWeight^{(h)}_{\gramsym_K}(\strv{K}) \\
&= 
\sum_{
\wproduction{\ntX}{\gramsym_1 \cdots \gramsym_K}{\rw}
}
\sum_{k=1}^K
\sum_{\str = \strv{1} \cdots \strv{k}}
\rw
\cdot
\preWeight^{(h)}_{\gramsym_1}(\strv{1})
\cdots
\preWeight^{(h)}_{\gramsym_k}(\strv{k})
\ \ \smashoperator{\sum_{\strt = \strv{k+1} \cdots \strv{K}}} \ \ 
\preWeight^{(h)}_{\gramsym_{k+1}}(\strv{k+1})
\cdots
\preWeight^{(h)}_{\gramsym_K}(\strv{K})
 \\
&= 
\sum_{
\wproduction{\ntX}{\gramsym_1 \cdots \gramsym_K}{\rw}
}
\sum_{k=1}^K
\sum_{\str = \strv{1} \cdots \strv{k}}
\rw
\cdot
\preWeight^{(h)}_{\gramsym_1}(\strv{1})
\cdots
\preWeight^{(h)}_{\gramsym_k}(\strv{k})
\left(
\preWeightTotal^{(h)}_{\gramsym_{k+1}}
\cdots
\preWeightTotal^{(h)}_{\gramsym_K}
\right) \\
&= 
\sum_{
\wproduction{\ntX}{\gramsym_1 \cdots \primeNT{\gramsym_k}}{\rw}
}
\sum_{\str = \strv{1} \cdots \strv{k}}
\rw
\cdot
\preWeight^{(h)}_{\gramsym_1}(\strv{1})
\cdots
\preWeight^{(h)}_{\primeNT{\gramsym_k}}(\strv{k}) \\
&=
\preWeight^{(h+1)}_{\primeNT{\ntX}}(\str)
\end{align}
\end{subequations}
}

\noindent The manipulation above proceeds in six steps:
\begin{enumerate}
\item \textbf{Tree-weight recursion for $\ntX$.} Each height-$\le h{+}1$ tree rooted at $\ntX$ applies some rule $\wproduction{\ntX}{\gramsym_1 \cdots \gramsym_K}{\rw}$ at the root with subtrees of height $\le h$.
\item \textbf{Border position.} Reparametrize the joint sum $\sum_\strt \sum_{\str\strt = \strv{1}\cdots\strv{K}}$ by introducing $k \in \{1, \ldots, K\}$ with $\str = \strv{1} \cdots \strv{k}$ and $\strt = \strv{k+1} \cdots \strv{K}$.
\item \textbf{Factor the $\strt$-side.} The factors $\preWeight^{(h)}_{\gramsym_1}(\strv{1}) \cdots \preWeight^{(h)}_{\gramsym_k}(\strv{k})$ are fixed by the $\str$-split, so they move out of the inner sum over $\strt = \strv{k+1} \cdots \strv{K}$.
\item \textbf{Collapse to totals.} We simplify the inner summation:
\begin{align*}
&\smashoperator{\sum_{\strt = \strv{k+1} \cdots \strv{K}}} \ \
\preWeight^{(h)}_{\gramsym_{k+1}}(\strv{k+1})
\cdots
\preWeight^{(h)}_{\gramsym_K}(\strv{K}) \\
&= \sum_{\strv{k+1}} \cdots \sum_{\strv{K}} \preWeight^{(h)}_{\gramsym_{k+1}}(\strv{k+1}) \cdots \preWeight^{(h)}_{\gramsym_K}(\strv{K}) \\
&= {\left[ \sum_{\strv{k+1}} \preWeight^{(h)}_{\gramsym_{k+1}}(\strv{k+1})\right]}\cdots
{\left[\sum_{\strv{K}} \preWeight^{(h)}_{\gramsym_K}(\strv{K})\right]} \\
&= \preWeightTotal^{(h)}_{\gramsym_{k+1}}
\cdots
\preWeightTotal^{(h)}_{\gramsym_K}
\end{align*}
\item \textbf{The prime rule.} The absorbed weight $\rw' = \rw \cdot \preWeightTotal^{(h)}_{\gramsym_{k+1}} \cdots \preWeightTotal^{(h)}_{\gramsym_K}$ matches the prime rule $\primeNT{\ntX} \to \gramsym_1 \cdots \gramsym_{k-1} \primeNT{\gramsym_k}$ (\cref{eq:prefix-grammar-optimized-borderline-exit}), so the joint sum over (original rule, $k$) becomes a pattern-match over such prime rules. Applying $\Phi(h)$ to $\gramsym_k$ rewrites $\preWeight^{(h)}_{\gramsym_k}(\strv{k})$ as $\preWeight^{(h)}_{\primeNT{\gramsym_k}}(\strv{k})$.
\item \textbf{Tree-weight recursion for $\primeNT{\ntX}$.} The result is the tree-weight recursion for $\preWeight^{(h+1)}_{\primeNT{\ntX}}(\str)$.
\end{enumerate}
This establishes $\Phi(h{+}1)$, completing the induction. Since every derivation tree in $\prefixGrammar$ has finite height, $\Phi(h)$ for all $h \ge 1$ recovers \cref{eq:stronger-statement}, and the proposition follows via the hook-up step.\footnote{Our proof implicitly uses $\preWeightTotal^{(h)}$ in the prime rule weights, i.e., the totals of $\prefixGrammar^{(h)}$ rather than $\prefixGrammar$. As $h \to \infty$, $\preWeightTotal^{(h)}_{\gramsym} \to \totalWeight{\gramsym}$, recovering $\prefixGrammar$'s rule weights; $\Phi(h)$ for all $h \!\ge\! 1$ then gives \cref{eq:stronger-statement}. A more formal induction use pairs $(\indicator{\gramsym \in \nonterms}, h)$.\looseness=-1}

\end{proof}

\clearpage
\subsection{Proof of \texorpdfstring{\cref{prop:optimized-grammar-bound}}{the grammar-size bound}}

\restatableGrammarSizeBound*
\begin{proof}
\label{prop:optimized-grammar-bound-proof}
According to \cref{def:prefix-grammar-optimized}, the prefix grammar $\prefixGrammar$ consists of the original rules $\rules$ plus the additional prefix rules $\primeRules$. We bound each contribution:
\begin{itemize}
\item The original rules $\rules$ contribute $|\grammar|$.
\item \Cref{eq:prefix-grammar-optimized-start,eq:prefix-grammar-optimized-exit} contribute $2 + 1 = 3$.
\item For \cref{eq:prefix-grammar-optimized-borderline-exit}:
\begin{itemize}
\item Each nullary rule (size $1$) produces no prefix rules.
\item Each unary rule (size $2$) produces one prefix rule ($k\!=\!1$) of size $2$, contributing $2$ to the size.
\item Each binary rule (size $3$) produces two prefix rules: one for $k\!=\!1$ of size $2$ and one for $k\!=\!2$ of size $3$, contributing $2+3=5$ to the size.
\end{itemize}
In these cases, the ratio of the prefix contribution to the original rule size is at most $5/3$ (attained by binary rules).
So the total contribution of \cref{eq:prefix-grammar-optimized-borderline-exit} is at most $\frac{5}{3}|\grammar|$.
\end{itemize}
Summing: $|\prefixGrammar| \le |\grammar| + 3 + \frac{5}{3}|\grammar| = \frac{8}{3}|\grammar| + 3$.\looseness=-1
\end{proof}

\end{document}